\documentclass{article}

\usepackage{arxiv}

\usepackage[utf8]{inputenc} 
\usepackage[T1]{fontenc}    
\usepackage{booktabs}       
\usepackage{amsfonts}       
\usepackage{nicefrac}       
\usepackage{amsmath} 
\usepackage{times}
\usepackage{subfig}
\usepackage{amsmath}
\usepackage{dsfont}
\usepackage{xr}
\usepackage{tikz}
\usepackage{pgfplots}
\usepackage{adjustbox}
\usepackage{bm, braket}
\usepackage{amsthm}
\usepackage{natbib}
\usepackage{doi}
\usepackage{dsfont}
\usepackage{mathtools} 
\usepackage{nicematrix}
\usepackage{extarrows} 
\usepackage{hyperref}
\usepackage{mathrsfs} 
\usepackage{enumitem}
\usepackage{subfig}
\usepackage{amsmath}
\usepackage{amssymb}
\usepackage{mathtools}
\usepackage{amsthm}
\usepackage{newtxmath}

\title{One-Bit Quantization and Sparsification for Multiclass Linear Classification via Regularized Regression}

\newif\ifuniqueAffiliation
\uniqueAffiliationtrue

\ifuniqueAffiliation 
\author{ Reza Ghane$^*$ \\
        Department of Electrical Engineering
	\\
	California Institute of Technology\\
	Pasadena, CA 91125 \\
	\texttt{rghanekh@caltech.edu} \\
	\And
        Danil Akhtiamov$^*$ \\
	Department of Computing and Mathematical Sciences\\
	California Institute of Technology\\
	Pasadena, CA 91125 \\\texttt{dakhtiam@caltech.edu} \\
	\AND
        Babak Hassibi \\
	Department of Electrical Engineering\\
	California Institute of Technology\\
	Pasadena, CA 91125 \\
	\texttt{hassibi@caltech.edu} \\
}
\else
\usepackage{authblk}
\fi

\renewcommand{\shorttitle}{One-Bit Quantization and Sparsification for Multiclass Classification}

\theoremstyle{plain}
\newtheorem{theorem}{Theorem}[section]

\newtheorem{lemma}[theorem]{Lemma}
\newtheorem{corollary}[theorem]{Corollary}
\theoremstyle{definition}

\theoremstyle{remark}

\newcommand{\calR}{\mathcal{R}}
\newcommand{\calN}{\mathcal{N}}
\newcommand{\bbE}{\mathbb{E}}
\newcommand{\bbP}{\mathbb{P}}
\newcommand{\bbR}{\mathbb{R}}
\newcommand{\msr}{\mathscr{r}}
\newcommand{\mutil}{\tilde{\mu}}

\newcommand{\ytil}{\tilde{y}}
\newcommand{\Ytil}{\tilde{Y}}

\newcommand{\phil}{\underline{\phi}_\lambda}
\newcommand{\phiu}{\bar{\phi}_\lambda}
\newcommand{\wl}{\underline{w}}
\newcommand{\wu}{\bar{w}}

\begin{document}
\maketitle
\def\thefootnote{*}\footnotetext{Equal contribution}
\maketitle

\begin{abstract}

We study the use of linear regression for multiclass classification in the over-parametrized regime where some of the training data is mislabeled. In such scenarios it is necessary to add an explicit regularization term, $\lambda f(w)$, for some convex function $f(\cdot)$, to avoid overfitting the mislabeled data. In our analysis, we assume that the data is sampled from a Gaussian Mixture Model with equal class sizes, and that a proportion $c$ of the training labels is corrupted for each class. Under these assumptions, we prove that the best classification performance is achieved when $f(\cdot) = \|\cdot\|^2_2$ and $\lambda \to \infty$. We then proceed to analyze the classification errors for $f(\cdot) = \|\cdot\|_1$ and $f(\cdot) = \|\cdot\|_\infty$ in the large $\lambda$ regime and notice that it is often possible to find sparse and one-bit solutions, respectively, that perform almost as well as the one corresponding to $f(\cdot) = \|\cdot\|_2^2$.

\end{abstract}

\section{Introduction}

While large neural networks, such as Large Language Models (LLMs), have increasingly become an indispensable part of almost every inference task, their impressive performance comes at the price of enormous scale. To name a few, Llama 2 and GPT-4 are each comprised of 65 billion and 1.76 trillion parameters, respectively. Such models suffer from the common issues pertaining to their size: they require very large memory for storage, can consume inordinate amounts of power, and it can be very inefficient to communicate them over communication networks. This calls for the development of efficient methods for model quantization and pruning, where quantization refers to the us of a small number of bits to store each model parameter, and where pruning refers to setting some of the parameter values to zero. While multiple papers suggesting such methods have appeared recently (see, e.g.. \cite{frantar2023massive, frantar2023qmoe, wang2023bitnet, shang2023pb} and the references therein), we have very little understanding of the theoretical limitations of such methods. In other words, one would like to have an answer to the following question: for a given model and data set, how much does one have to sacrifice in performance if instead of looking at all potential solutions one looks only at, say, solutions of a given sparsity, or solutions where the weights are quantized to a certain level, such as one-bit? Of course, this question remains wide open for general supervised machine learning problems. As a first step in this direction, we will investigate it in the context of multiclass linear classification using regularized linear regression. 

Multiclass classifiers play a crucial role in numerous machine learning applications. In fact, it can be argued that the power of deep learning manifests itself most prominently in the multiclass setting. Many image databases, for example \cite{imagenet}, contain tens of classes which any classifier will need to contend with. Multiclass classification further arises in problems with categorical outputs, such as natural language processing \cite{ilya}, in reinforcement learning \cite{jang16}, where the agent has to choose the correct action among a large set of possible ones, in recommendation systems \cite{cas16}, and in many others. 

Despite the widespread use and success of multiclass classification, we still have a rather poor understanding of the theory behind their performance. There are many pertinent questions for which we do not have answers. For example, what is the classification error? How does it depend on the number of classes? How does it depend on the loss function and the regularizer being used in the training? What are the best algorithms for finding good sparse or good compressed classifiers? Can we determine the best possible generalization performance among all solutions of given sparsity, or among all solutions where each parameters is quantized to a single bit? Can we efficiently learn these from data? What is the best one can do on a dataset with label corruption? Currently, we are quite far from satisfactorily answering any of these questions. 

A common feature of modern successful multiclass classifiers is that they are highly over-parametrized. In fact, the number of parameters that needs to be trained is often orders of magnitude larger than the number of training samples. In this setting, models can easily interpolate the training data and classical learning theory would raise the concern of ``over-fitting" and loss of generalization performance. In practice, however, the opposite has been often observed---increasing the parameters in the model improves generalization performance, a phenomenon referred to as ``double descent" \cite{belkin2020two}. It is now largely agreed upon that, the reason why over-fitting does not arise is that the (stochastic) gradient-based methods that are used to train over-parametrized models possess ``implicit bias" and ostensibly find ``good" interpolating solutions \cite{gunasekar2017implicit,azizan2019stochastic,azizan2021stochastic,gunasekar2018characterizing,gunasekar2018implicit}. Nonetheless, the precise relationship between implicit bias and classification error is not well understood. For example, it has been observed that changing the training algorithm can change the implicit bias, and thereby the classification error \cite{azizan2021stochastic}. However, the relationship between implicit bias and generalization is not known and is not clear what type of implicit bias will lead to good generalization performance. Furthermore, the theoretical analysis of double descent (see, e.g., \cite{belkin2020two,deng2022model} becomes less responsive when it comes to having a training dataset that contains mislabeled data. Given the described challenges, in this paper we will focus our study on a special case that is much more amenable to analysis, namely multiclass linear classifiers, hoping that the insights obtained will be informative for more general multiclass classifications settings.

In particular, we assume that the data is sampled from a Gaussian Mixture Model (GMM) with equal class sizes where a proportion $c$ of the training labels is corrupted for each class. Since we consider the over-parametrized regime, to avoid overfitting the mislabeled data we propose regularized linear regression with a regularizer $\lambda f(w)$, for some convex function $f(\cdot)$. Even this setting is rather challenging to analyze and so we focus our analysis on the regime of strong regularization, i.e., large $\lambda$, where we can obtain precise expressions for the classification error of the solution corresponding to an arbitrary convex regularizer $f(\cdot)$ and an arbitrary number of classes. This is perhaps the most interesting case, because (as we show rigorously in the sequel) for $f(\cdot) = \|\cdot\|_2^2$ it results in the best classification performance. Further, for $f(\cdot) = \|\cdot\|_1$ we show that large $\lambda$ results in highly sparse sparse solutions with good performance, and that for $f(\cdot) = \|\cdot\|_\infty$ large $\lambda$ results in solutions that can be compressed all the way to a single bit per parameter with little loss of performance. We also present numerical results validating the theory. Put together, our results suggest that in the over-parametrized regime, one can often find sparse or one-bit solutions without having to compromise much on performance even though the underlying data model has no inherent sparsity structure. We believe this observation might have significant theoretical implications for general over-parametrized models due to the common intuition that large neural networks tend to behave similarly to their first-order linear approximations.

\section{Related works}

In the case of binary classification, there has been a recent surge of results that provide a sharp analysis for a variety of methods tailored to different models (see, e.g., \cite{thrampoulidis2018precise,HH17,candessur2018,cansur19,ka20,salehi2019impact,tpt20,zat19,mrsy19,mklz20,lolas20,akhtiamov2023regularized} and the references therein). These works typically pose the over-parameterized binary classification problem as an optimization problem and employ the Convex Gaussian Min-Max Theorem (CGMT) \cite{thrampoulidis2015regularized,thrampoulidis2018precise,stojnic2013framework,gordon1985some} to obtain formulas for the classification error that involve solutions to a system of non-linear equations (in a small number of variables) that often do not admit closed-form expressions. These results follow a long line of work that deals with obtaining {\em sharp} high-dimensional asymptotics of convex optimization-based estimators. 

While the above works have significantly advanced our understanding of binary classification in the over-parameterized regime, they do not shed much light on the multiclass setting. Four interesting recent papers that study the multiclass setting include \cite{wang21, subramaniangeneralization, loureiro2021learning, thrampoulidis2020theoretical}. 

\cite{wang21} establishes that in certain over-parameterized regimes the solution to a multiclass supprt vector machine (SVM) problem is identical to the one obtained by minimum-norm interpolation of one-hot encoded labels (equivalently, that gradient descent on squared loss leads to the same solution as gradient descent on cross-entropy loss, since the solution is determined by the implicit bias of these algorithms, not of the loss). In addition, \cite{wang21} shows that the multiclass problem with finitely many classes reduces to a finite set of pairwise competitions, all of which must be won for multiclass classification to succeed. \cite{subramaniangeneralization} establishes that multiclass classification qualitatively behaves like binary classification, as long as there are not too many classes. It further makes note that the key difference from the binary classification setting is that, as the number of classes increases, there are relatively fewer positive training examples of each class, making the multiclass problem ``harder" than the binary one. We remark that the settings of \cite{wang21} and \cite{subramaniangeneralization} differ from ours: they consider a 1-sparse noiseless model of the labels, and a bi-level ensemble model for the different classes. We instead consider a Gaussian mixture model for the different classes. 

\cite{thrampoulidis2020theoretical} provided the first precise analysis of linear multi-class classification in the high-dimensional regime. Their work revealed that the classification accuracy depends on certain quantities, such as the correlation between the trained weights for each class, as well as the correlation between the trained weights and the means of the underlying distributions for each class. Computing these quantities is quite challenging and \cite{thrampoulidis2020theoretical} is only able to do so for pure least-squares classifiers and for least-squares classifiers with $f(\cdot) = \|\cdot\|_2^2$ regularization.\footnote{\cite{thrampoulidis2020theoretical} uses CGMT and a full analysis of the multiclass classification problem requires a matrix generalization of the CGMT which currently does not exist.} Also, they did not consider any label corruption, which would further impede the analysis. 

\cite{loureiro2021learning} use the celebrated replica method as intuition in order to derive an AMP (approximate message passing) sequence capturing the performance of the multiclass classification problem for an arbitrary convex loss function and an arbitrary convex regularizer. This way, \cite{loureiro2021learning} expresses the classification error in terms of a few scalar quantities that can be determined from a low-dimensional system of saddle-point equations. Their analysis shows that, if the means of the GMM are sparse, then strong $\ell_1$ regularization improves performance. However, we show that, even when the means of the underlying GMMS are not sparse, strong $\ell_1$ regularization can lead to sparse solutions without sacrificing too much on classification performance. Furthermore, \cite{loureiro2021learning}  does not study the use of regularization for model compression and also do not consider any corruption in the labels. Finally, the methods developed in our work use CGMT and are therefore entirely different, in addition to being quite explicit in the large $\lambda$ regime. From this vantage point, they could be of independent interest in their own right.

\section{Setup and preliminaries}

\subsection{The Gaussian mixture model with corruption}

We will assume that we have $k$ classes, $n$ training samples, and that the training data is given by $\{(x_i,y_i)\}_{i=1}^n$, where each $x_i\in\bbR^d$ is an input, and $y_i$ is the corresponding (potentially corrupted) class label. We will further assume that the $x_i$ are generated in an iid fashion from a Gaussian mixture model (GMM), the data points are drawn uniformly from each class and the labels are corrupted in a uniform manner at a fixed rate $c < \frac{1}{k-1}$. In other words, denoting the true label designating the class to which $x_i$ belongs by $\tilde{y_i}$,
\begin{align}
& \mbox{Prob}(\ytil_i = \ell) = \frac{1}{k} \nonumber \\
& \mbox{Prob}(y_i = \ell) = \begin{cases} 1 - c, &  \ell = \ytil_i \\ \frac{c}{k-1}, &   \ell \ne \ytil_i \end{cases}  ~~~\mbox{and}~~~x_i\sim{\cal N}(\mu_{\tilde{y_i}},\sigma^2I),
\label{eq:gmm}
\end{align}
i.e. the inputs for each class $\ell$ are drawn from an isotropic multi-variate Gaussian distribution with mean $\mu_\ell$ and variance $\sigma^2I$. We will find it useful to aggregate the mean vectors into the $k\times d$ matrix
\begin{equation*}
    M^T = \left[ \begin{array}{cccc} \mu_1 & \mu_2 & \ldots & \mu_k \end{array} \right] 
\end{equation*}
If we further define the $n\times d$ and $n\times k$ matrices
\begin{align*}
    & X^T = \left[ \begin{array}{cccc} x_1 & x_2 & \ldots & x_n \end{array} \right], \quad
    \Ytil^T = \left[ \begin{array}{cccc} e_{\ytil_1} & e_{\ytil_2} & \ldots & e_{\ytil_n} \end{array} \right] \\
    & Y^T = \left[ \begin{array}{cccc} e_{y_1} & e_{y_2} & \ldots & e_{y_n} \end{array} \right]
\end{align*}
where $e_\ell\in\bbR^k$ is the $\ell$-th standard basis vector in $k$-dimensional space. We may write
\begin{equation}
    X = \Ytil M+A,
    \label{eq:X}
\end{equation}
where $A\in\bbR^{n\times d}$ has iid ${\cal N}(0,\sigma^2)$ entries. 

As one may expect, the performance of multiclass classification for the GMM model highly depends on the relative positions of the means $\mu_\ell$. While it is possible to analyze general $\mu_\ell$, we will assume that the $\mu_\ell$ are generated acoording to Assumpion A2 in Section \ref{subsec:assumption}. This assumption implies that $\mu_\ell^T \mu_{\ell'} \approx r\|\mu_\ell\|\cdot\|\mu_{\ell'}\|$, for every $\ell\neq\ell'$, i.e., that the mean vectors are pairwise equiangular. 

\subsection{Why considering GMMs is not too limiting: Gaussian Universality}

While it might seem at first that the GMM assumption is very restrictive, due to Gaussian universality, this is not really the case. Put in our context, intuitively this says that for a large class of mixtures of $k$ distributions modeling the classes, replacing those with the GMM having the same means and covariances for each class will lead to asymptotically equal classification error as $n$ grows large. An interested reader can find a rigorous exposition and detailed related discussions in \cite{dandi2023universality}. To further illustrate this point, we conducted experiments exploring the performance of regularized linear models trained on MNIST and observed the performance to be very similar qualitatively to our predictions for GMMs (see Section \ref{sec: exp} in the Appendix).

\subsection{CGMT}

The CGMT framework starts with a so-called primary optimization (PO) problem
\begin{equation}
    \min_{u\in S_u}\max_{v\in S_v} u^TGv+\psi(u,v),
    \label{eq:po}
\end{equation}
where $u\in\bbR^d$, $v\in\bbR^n$, $S_u$ and $S_v$ are convex sets, $\psi(\cdot,\cdot)$ is convex in its first argument and concave in its second one, and $G\in\bbR^{d\times n}$ has iid $\calN(0,1)$ entries. To analyze the (PO), the CGMT framework introduces the so-called auxilliary optimization (AO) problem 
\begin{equation}
    \min_{u\in S_u}\max_{v\in S_v} \|u\|_2g^Tv+\|v\|_2h^Tu+\psi(u,v),
    \label{eq:ao}
\end{equation}
where $g\in\bbR^n$ and $h\in\bbR^d$ have iid $\calN(0,1)$ entries. Roughly speaking, the CGMT frameowrk states that, if, as $n\rightarrow\infty$ while keeping $\frac{d}{n}$ fixed, the optimum of (\ref{eq:ao}) concentrates to some value $c$, say, then the optimum of (\ref{eq:po}) concentrates to the {\em same} value. More importantly, if any Lipschitz function of the solution ${\hat u}_{AO}$ of (\ref{eq:ao}), say its $\ell_2$-norm, concentrates to some value, then the same Lipschitz function of the corresponding solution ${\hat u}_{PO}$ of (\ref{eq:po}) will concentrate to the {\em same} value. In a nutshell properties of the solution to the (PO) can be inferred from properties of the solution to the (AO). The advantage of this approach is that analyzing the (AO) is usually much simpler than the (PO) since it contains two random vectors, rather than an entire random matrix. Details and much more rigorous statements, can be found in \cite{thrampoulidis2015regularized,thrampoulidis2018precise}. 

\subsection{Classification error for linear classifiers}

\label{subsection: classif_err}

We will be considering linear classifiers. Thus, let $W\in\bbR^{d\times k}$ be the matrix of $k$ weight vectors $w_\ell\in\bbR^d$, for $\ell = 1,2,\ldots ,k$. Given a new input vector $x$, the linear classifier will estimate its label according to:
\begin{equation}
    {\hat y} = \mbox{arg}\max_{\ell\in [k]}w_\ell^Tx.
\end{equation}
The {\em per-class} classification error is defined as 
\begin{equation}
    P_{e|\ell} = \mbox{Prob}({\hat y}\neq y|y = \ell), ~~~\ell = 1,2,\ldots k
\end{equation}
and the {\em total} classification error as
\begin{equation}
    P_e = \mbox{Prob}({\hat y}\neq y) = \frac{1}{k}\sum_{\ell=1}^kP_{e|\ell}.
    \label{eq:gen_error}
\end{equation}
For $\ell  = 1,2,\ldots k$, let $S_\ell\in{\bbR^{(k-1)\times (k-1)}}$ denote the symmetric positive semidefinite (psd) matrix whose $(i,j)$-th entry, $i,j\ne \ell$ is given by $(S_\ell)_{ij} = (w_\ell-w_i)^T(w_\ell-w_j)$ and let $S_\ell^{1/2}$ denote its psd square root. Further, let $t_\ell\in\bbR^{k-1}$ denote a vector whose $i$-th entry for $i \ne \ell$ is $(t_\ell)_i = (w_i-w_\ell)^T\mu_\ell$. Then \cite{thrampoulidis2020theoretical}, shows that 
\begin{equation}
    P_{e|\ell} = 1-\bbP(S_\ell^{1/2}g\geq t_\ell),
    \label{eq:conditional}
\end{equation}
where $g\in\bbR^{k-1}$ has iid $\calN(0,\sigma^2)$ entries. 

We remark that probabilities of the form (\ref{eq:gen_error}) and (\ref{eq:conditional}), can be straightforwardly computed via Monte Carlo simulation once $S_\ell^{1/2}$ and $t_\ell$ are known. However, due to the symmetry that arises from having pairwise equiangular means $\mu_\ell$ and isotropic variances $\sigma^2 I$, it turns out that 
\begin{equation} \label{eq: class_error}
P_e = P_{e|\ell} = Q_k\left(\frac{\mu_{\ell}^T(w_\ell - w_{\ell'})}{\sigma\|w_\ell - w_{\ell'}\|}\right),
\end{equation}
for any $\ell\neq\ell'$ where
\begin{align}
& Q_k(a) = \bbP((I + \frac{\mathds{1}\mathds{1}^T}{1+\sqrt{k}}) g \ge \mathds{1} \sqrt{2}a) \\ \nonumber
\end{align}

\subsection{Overview of approach}
The linear multi-class classifier just described needs a set of weight vectors for each class. These weights should be learned from the training data. In what follows, we will assume that the estimator of the weights is found from solving the following optimization problem
\begin{align}
    & \min_W \left[ \|XW-Y\|_F^2+\lambda \sum_{\ell=1}^kf(w_\ell) \right] =  \sum_{\ell=1}^k\min_{w_\ell}\left[ \|Xw_\ell-Y_\ell\|_2^2+\lambda f(w_\ell)\right] 
    \label{eq:est}
\end{align}
where $Y_\ell\in\bbR^{n}$ denotes the $\ell$-the column of $Y$, i.e., the $i$-th entry of $Y_\ell$ is one if the $i$-th data point is labelled $\ell$ and zero otherwise, $f(\cdot)$ is a convex regularizer of the weights, and $\lambda >0$ is the regularization parameter. Note that, in principle, the statistics for each {\em separate} $w_l$ can be found by analyzing $\displaystyle \min_{w_\ell}\|Xw_\ell-Y_\ell\|_2^2+\lambda f(w_\ell)$ using CGMT. However, this approach cannot capture the cross-correlations between $w_{\ell}$ and $w_{\ell'}$ for $\ell \ne \ell'$. Thus, we adopt a different approach in the present paper: we still use CGMT, but only for the sake of simplifying (\ref{eq:est}) under assumption that $\lambda$ is large. After making the necessary simplifications, we proceed to analyze the simplified PO in the large $\lambda$ regime. We can capture the cross-correlations now because we can analyze the simplified PO directly. To proceed according to the described plan, we need to transform the objective first:

\begin{lemma}\label{lem: tran_PO}
Let \begin{align} \label{eq: POl}
        & \tilde{\Phi}_{\ell,s,t}(w) = \|\tilde{A}w\|^2 + \frac{n}{k}\sum_{\mathscr{r} \ne \ell }(w^T\tilde{\mu}_{\mathscr{r}}-\frac{c}{k-1})^2+ \frac{n}{k}\left[(w^T\mutil_{\ell}-(1-c))^2+(\sqrt{\frac{k}{n}}w^Ta-s)^2\right]+ \nonumber \\
       & + \frac{n}{k}(\sqrt{\frac{k}{n}}w^Tb-t)^2 + \lambda f(w)
\end{align}
Where $\mutil_\ell = \mu_\ell + \sqrt{\frac{k}{n}}\eta_l$ and $\tilde{A} \in \mathbb{R}^{(n - k - 2) \times d}, a, b, \eta_\ell \in \mathbb{R}^d$are i.i.d. $\mathscr{N}(0, \sigma^2)$ and $s, t$ are scalars defined in (\ref{eq: st}).
Assume that $X$ is distributed according to (\ref{eq:X}). Then the distributions of the solutions for the following two optimization problems are identical:
\begin{equation} 
\min_{w_\ell, w_{\ell'}} \|Xw_\ell-Y_\ell\|_2^2+ \|Xw_{\ell'}-Y_{\ell'}\|_2^2 + \lambda f(w_\ell) +  \lambda f(w_{\ell'})
\label{eq: initial_PO}
\end{equation}
\begin{equation} 
\min_{w_\ell, w_{\ell'}} \Phi_{\ell,s,t}(w_\ell) + \Phi_{\ell',t,s}(w_{\ell'}) 
\label{eq: transformed_PO} 
\end{equation}
\end{lemma}
The above lemma guarantees that, any scalar characteristics of the joint distribution of $w_\ell$ and $w_{\ell'}$ in (\ref{eq:est}) can be found from the corresponding scalar characteristics in (\ref{eq: transformed_PO}). Note that characterising only such pairwise joint distributions is enough according to subsection \ref{subsection: classif_err}. Let us now introduce the following notation: 
\begin{align}
    & F_{\ell}(w_\ell) =  \frac{n}{k}\left[\sum_{\msr \ne \ell }(w_\ell^T\mutil_{\msr}-\frac{c}{k-1})^2 +(w_\ell^T\mutil_{\ell}-(1-c))^2 + (\sqrt{\frac{k}{n}}w_\ell^Ta-s)^2+(\sqrt{\frac{k}{n}}w_\ell^Tb-t)^2 \right]   \\\nonumber
    & \phil^{\ell, \rho} (w_\ell):= \rho\|w_\ell\|_2^2 + F_{\ell}(w_\ell) + \lambda f(w_\ell) \\ 
    & \phiu^{\ell} (w_\ell):= \sigma^2n\|w_\ell|^2 + F_{\ell}(w_\ell) + \lambda f(w_\ell) 
    \label{eq : phi}
\end{align}
We first proceed to prove that (\ref{eq: transformed_PO}) is sandwiched between $\displaystyle \min_{w_\ell, w_{\ell'}} \phil^{\ell, \rho} (w_\ell) + \phil^{\ell', \rho} (w_{\ell'})$ and $\displaystyle \min_{w_\ell, w_{\ell'}} \phiu^{\ell} (w_\ell) + \phiu^{\ell'} (w_{\ell'})$ for small enough $\rho > 0$ and, moreover, $\displaystyle \min_{w_\ell, w_{\ell'}} \phil^{\ell, \rho} (w_\ell) + \phil^{\ell', \rho} (w_{\ell'})$ and $\displaystyle \min_{w_\ell, w_{\ell'}} \phiu^{\ell} (w_\ell) + \phiu^{\ell'} (w_{\ell'})$ take the same value as $\lambda \to \infty$ (cf. Appendix \ref{sec: sandwich}). We use this to rigorously prove that solving either $\phil^{\rho}$, for small enough $\rho > 0$, or $\phiu$ yields the same classification error as the solution of (\ref{eq: transformed_PO}) when $n \to \infty$ (cf. Theorem \ref{thm : solutions_match}).  After that, we solve $\phil^{\ell, \rho}$ and $\phiu^{\ell}$ to prove Theorems \ref{thm: l2} , \ref{thm: l1} and \ref{thm: linf}. Finally, we use the result of Theorem \ref{thm: l2} to deduce Theorem \ref{thm: optimal_reg}.

\section{Main Results}

\subsection{Assumptions and notation}
\label{subsec:assumption}

We start with the assumptions. (A2) is not necessary but adding it makes the results much more explicit. 
\begin{enumerate}[label=(A\arabic*)]
    \item The data is generated from a GMM with means $\mu_1, \dots, \mu_k$ and covariance matrix $\sigma^2I$ for each class.
    \item Each mean $\mu_\ell$, $\ell = 1, \dots, k$ is i.i.d. $\mathscr{N}(0,1)$ and matching coordinates of different means are correlated according to $\mathbb{E}(\mu_\ell)_i(\mu_{\ell'})_i = r$ for some $-1 \leq r \leq 1$ and all $i = 1, \dots, d$ whenever $\ell \ne \ell'$. 
    \item The ratio $\frac{d}{n}$ is fixed as $n \to \infty$.
    \item The regularizer $f: \mathbb{R}^d \to \mathbb{R}_{+}$ is convex and satisfies $f(w) \ge q(d)\|w\|_2$ for some $q: \mathbb{N} \to \mathbb{R}_+$.
\end{enumerate}

Finally, also denote
\begin{align}
    &  s = \frac{1}{2}\left[\sqrt{\frac{(k-2)(2c-\frac{kc^2}{k-1})}{k-1}} + \sqrt{\frac{k(2c-\frac{kc^2}{k-1})}{k-1}} \right] \nonumber \\ 
&  t = \frac{1}{2}\left[\sqrt{\frac{(k-2)(2c-\frac{kc^2}{k-1})}{k-1}} - \sqrt{\frac{k(2c-\frac{kc^2}{k-1})}{k-1}} \right]
\label{eq: st}
\end{align}
\subsection{The optimal regularizer and $\lambda$}

\begin{theorem}
    Assume that (A1)-(A3) hold and that either $\sigma^2 = o(n)$ or $c = 0$. Then taking $f(\cdot) = \|\cdot\|_2^2$ and $\lambda \to \infty$ minimizes the classification error $P_e$ in (\ref{eq:gen_error}) among regularizers satisfying (A4).  
    \label{thm: optimal_reg}
\end{theorem}

Informally, the theorem above says that if one wants to use regularized linear regression to solve a multiclass classification task in the presence of mislabeled data, the best one can do is to set $f(\cdot) = \|\cdot\|_2^2$ and take a very large $\lambda$.  This result is part of the reason for why we focus on strong regularization (large $\lambda$) in this paper. 

\subsection{The Sandwich Theorem for multiclass linear classification in the large $\lambda$ regime}
To tackle the problem of analyzing (\ref{eq: transformed_PO}) in the large $\lambda$ regime, we "sandwich" it between two optimization problems that are much easier to solve. After that, we prove that solutions of (\ref{eq: transformed_PO}) and the problems that "sandwich" it match asymptotically. Recall the "sandwiching" optimization problems are $\displaystyle \min_{w_\ell, w_{\ell'}} \phil^{\ell, \rho} (w_\ell) + \phil^{\ell', \rho} (w_{\ell'})$ for small enough $\rho > 0$ and $\displaystyle \min_{w_\ell, w_{\ell'}} \phiu^{\ell} (w_\ell) + \phiu^{\ell'} (w_{\ell'})$.  Denote the classification error obtained from (\ref{eq: transformed_PO}) by $P_{e,\lambda}$ and that obtained by the upper bound sandwich problem by ${\bar P}_{e,\lambda}$.  Then the following instrumental theorem shows that finding the classification error for (\ref{eq: transformed_PO}) reduces to finding the classification error for the upper bound problem. 

\begin{theorem}
     \label{thm : solutions_match}
     Consider any $f$ satisfying (A4). For any $\epsilon > 0, \delta>0$, there exists a $N_0,\lambda_0(N_0)$, such that for every $n>N_0, \lambda > \lambda_0(N_0)$
     $$
     \bbP(| P_{e,\lambda}-{\bar P}_{e,\lambda} | < \epsilon ) > 1-\delta
     $$
\end{theorem}
\subsection{Main Theorems}
We now present a rigorous characterization of the performance of multiclass classification based on regularized linear regression. We start with $f(\cdot) = \|\cdot\|_2^2$: 

\begin{theorem} (Ridge Regression)
    Assume that (A1)-(A3) hold. Then in the case of  $f(\cdot) = \|\cdot\|_2^2$ we have that $\displaystyle \lim_{\lambda \to \infty}P_{e,\lambda}$ is equal to the following with probability approaching $1$ as $n \to \infty$:
    $$\lim_{\lambda \to \infty} Q_k(\frac{d(\zeta - \gamma)(1 - r)}{2d(\zeta - \gamma)^2(1 + \frac{\sigma^2k}{n} - r) + 2(\alpha - \beta) ^ 2\sigma ^ 2d})$$
    where $\zeta, \gamma, \alpha, \beta, \Delta$ are closed-form functions of $\lambda$ omitted here for brevity and written in full in the Appendix in Theorem \ref{thm: l2_complete}
    
\label{thm: l2}
\end{theorem}
As noted in Theorem \ref{thm: optimal_reg}, when $\lambda$ is large, the above classifier achieves optimal classification error among all regularized linear regression classifiers. However, it may not have other desirable properties, such as yielding sparse solutions.  Thus, one might wonder about the classification error and the sparsity rate arising from the application of $f(\cdot) = \|\cdot\|_1$. The following theorem addresses this question:

\begin{theorem}\label{thm: l1} (LASSO Regression)
Under the assumptions (A1)-(A3), for any $\epsilon>0$, there exists $\lambda_0$ such that  for all $\lambda \ge \lambda_0$, the classification error corresponding to $f(\cdot) = \|\cdot\|_1$ can be described as follows as $n \to \infty$:
$$
|P_{e,\lambda}  - Q_k (\frac{n\sqrt{d}(1-r) (\gamma_2 - \gamma_1) R(\Xi) }{k|\Delta(\Xi)|}) | < \epsilon
$$
Where $\Xi, \Delta(\Xi)$ and $R(\Xi)$ are explicit functions of $\lambda$ and other parameters whose definitions are provided in the Appendix (cf. Theorem \ref{thm: l_1}) and $\gamma_1, \gamma_2, \gamma_3, \gamma_4$ are the solutions to the following scalar optimization problem:
\begin{align*}
&\max_{\gamma_{1},\gamma_{2},\gamma_{3},\gamma_{4}} -\frac{d (\Xi^2 + \lambda^2) R(\Xi)}{4 n\sigma^2} + \frac{ d \Xi \lambda }{2n\sigma^2 \sqrt{2\pi}} \exp(-\frac{\lambda^2}{2\Xi^2})+ \\ &+\frac{n}{k} \left[-c\gamma_1-(k-1)\frac{\gamma_1^2}{4}-(1-c)\gamma_2-\frac{\gamma_2^2}{4}-s\gamma_3-\frac{\gamma_3^2}{4}-t\gamma_4-\frac{\gamma_4^2}{4}\right]
\end{align*}
Moreover, the classifier is $\left \lceil d R(\Xi) \right \rceil = \left \lceil  2d Q(\frac{\lambda}{\Xi})\right \rceil$-sparse.
\end{theorem}

We also deduce the following corollary. 

\begin{corollary} \label{cor: l1}
    Assume (A1)-(A3) hold. Then the fraction of nonzero entries of the optimal $w$ for the case $f(\cdot) = \|\cdot\|_1$ goes to zero as $\lambda \to \infty$.
\end{corollary}
This means that as $\lambda \to \infty$, the classifier has a classification error of $Q_k(0) = \frac{k-1}{k}$ which is simply that of random guessing. Thus, the optimal classifier for $f(\cdot) = \|\cdot\|_1$ occurs for some intermediate value of $\lambda$. Another desirable property one might try to achieve is compressibility. It turns out that $f(\cdot) = \|\cdot\|_{\infty}$ pushes a big proportion of the weights to concentrate around $\pm \|w\|_{\infty}$ as $\lambda$ grows large. Formally, we have the following theorem and corollary:

\begin{theorem}
    \label{thm: linf} ($\|\cdot\|_{\infty}$ - Regularization) \\
    Under the assumptions (A1)-(A3), the classification error $P_{e,\lambda}$ corresponding to $f (\cdot) = \|\cdot\|_{\infty}$ can be described as follows in the large $\lambda$ regime with probability approaching $1$ as $n \to \infty$:
    $$\lim_{\lambda \to \infty} Q_k ( \frac{\sqrt{d} (1-r) (\gamma_2 - \gamma_1) (1-R(\Xi)) }{2k\sigma^2|\Delta|(\Xi)})
    $$
Where $\Xi, \Delta(\Xi)$ and $R(\Xi)$ are explicit functions of $\lambda$ and other parameters whose definitions are provided in the Appendix (cf. Theorem \ref{thm: linf_complete}) and $\gamma_1, \gamma_2, \gamma_3, \gamma_4$ are the solutions to the following scalar optimization problem:
\begin{align*}
        & \min_{\delta \ge 0}  \max_{\gamma_1,\gamma_2,\gamma_3,\gamma_4} \delta +\frac{nd\delta^2 \sigma^2  R(\Xi)}{\lambda^2} - (1-R) \frac{d \Xi^2}{4 n \sigma^2} - \frac{d\delta\Xi}{\lambda\sqrt{2\pi}}\exp(-\frac{2n^2 \sigma^4 \delta^2}{\Xi^2 \lambda^2}) +  \\
        &  + \frac{n}{k} \left[-c\gamma_1-(k-1)\frac{\gamma_1^2}{4} -(1-c)\gamma_2-\frac{\gamma_2^2}{4}-s\gamma_3-\frac{\gamma_3^2}{4}-t\gamma_4-\frac{\gamma_4^2}{4} \right] 
\end{align*} 
Moreover, $\zeta$ of the coordinates of the weights are equal to $\frac{\delta}{\lambda}$ and $\zeta$ of the weights are equal to $-\frac{\delta}{\lambda}$ and the rest of them take values between $-\frac{\delta}{\lambda}$ and $\frac{\delta}{\lambda}$ with $\zeta = \left \lceil \frac{R(\Xi)}{2} \right \rceil = \left \lceil Q(\frac{2\delta\rho}{\lambda \Xi}) \right \rceil$
\end{theorem}

The aforementioned theorem implies the following corollary, that says that in the limit $\lambda \to \infty$ exactly half of the weights concentrate at $-\frac{\delta}{\lambda}$ and the remaining half concentrate at $\frac{\delta}{\lambda}$. Since the classification error $P_e$ does not change if one rescales all weights simultaneously, it means that one can replace the weights of the solutions for $f(\cdot) = \|\cdot\|_{\infty}$ by their signs in the large $\lambda$ regime without affecting the performance.

\begin{corollary} \label{cor: linf}
        Assume (A1)-(A3) hold and let $\zeta$ be defined as in the theorem above. Then $\zeta$ goes to $\frac{1}{2}$ as $\lambda$ goes to $\infty$: 
     $$\lim_{\lambda \to \infty} \zeta = \frac{1}{2} $$
\end{corollary}
    
\section{Numerical Simulations}

To validate our theory numerically, we plot the classification error obtained by solving the regularized linear regression problem, as well as the predictions coming from the scalarizations of $\bar{\phi}$ for three different tuples of $(d, n, r, k, \sigma, c)$ and for 
$f(\cdot) = \|\cdot\|_2^2$, $f(\cdot) = \|\cdot\|_1$ and $f(\cdot) = \|\cdot\|_\infty$ against the corresponding value of $\lambda$. As expected, the predictions deviate from the error quite drastically for the small values of $\lambda$, start approaching it closer for intermediate values, and then match it almost perfectly as $\lambda$ grows large.  To calculate the classification error (which we refer as "true error"), we solved (\ref{eq:est}) using CVXPY \cite{diamond2016cvxpy, agrawal2018rewriting} and evaluated the corresponding error via (\ref{eq:gen_error}). We performed the latter procedure multiple times and then averaged the answer over the runs to mitigate the noise. For deriving the predictions, we used Theorems \ref{thm: l2}, \ref{thm: l1} and \ref{thm: linf}. For plotting, we used the closed-form expressions provided in the theorem for the case of $f(\cdot) = \|\cdot\|_2^2$, employed scipy.optimize.fmin initialized at $(\gamma_1, \gamma_2, \gamma_3, \gamma_4) = (0.1, 0.1,0.1,0.1)$ for $f(\cdot) = \|\cdot\|_1$ and ran a grid search over $\delta$ using scipy.optimize.fmin initialized at $(\gamma_1, \gamma_2, \gamma_3, \gamma_4) = (0.1, 0.1,0.1,0.1)$ each time for $f(\cdot) = \|\cdot\|_\infty$  (the value for the initializations was chosen arbitrarily and just turned out to always yield the correct parameters). We put and describe the plots for one tuple of $(d,n,k,c,r, \sigma)$ in the next subsections and the remaining plots can be found in Section \ref{sec: exp} of the Appendix.  We indicate the values for the corresponding tuples of $(d, n, r, k, \sigma, c)$ in the descriptions of the plots. 
 
\subsection{$f(\cdot) = \|\cdot\|_2^2$} The numerical experiments we conducted for $f(\cdot) = \|\cdot\|_2^2$ can be foind in Fig.\ref{fig:l2_500}. As predicted, the true error decreases as $\lambda$ increases and starts matching the values predicted via $\bar{\phi}$ as $\lambda$ grows large. In addition, the plot suggests that solving $\bar{\phi}$ provides a lower bound for the true classification error over the entire $\lambda$ domain. We leave theoretical analysis of this phenomenon to future work. 
\begin{figure}[htp]
    \centering
    \includegraphics[width = 3.25 in]{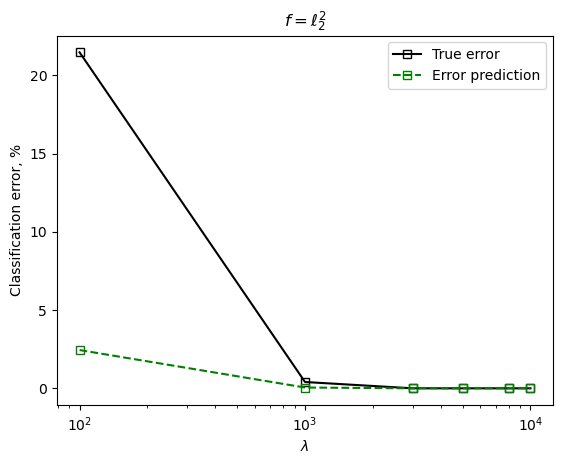}
    \caption{We took $d = 750$, $n = 500$, $k = 5$, $r = 0.8$, $c =0.3$ and $\sigma = 1$. The prediction underestimates the true error for smaller values of $\lambda$ but, as expected, matches it for larger ones. }
    \label{fig:l2_500}
\end{figure}

\subsection{$f(\cdot) = \|\cdot\|_1$}
The results of our experiments for $f(\cdot) = \|\cdot\|_1$ can be found in Fig.\ref{fig:l1_750}. 
The plot on the top depicts three lines: the true classification error, the classification error of the sparsified solution, and the classification error prediction from Theorem \ref{thm: l2}. To obtain the sparsified solution, we predicted sparsity rate $R$ from our analysis of ${\bar \phi}$, and then took the true solution of ($\ref{eq:est}$) and set the smallest $d(1-R)$ values of each $w_\ell$ to zero. The classification error predicted by analysis of ${\bar \phi}$ underestimates the true error for smaller values of $\lambda$ but matches it for larger ones. The sparsified solution performs worse than the true solution at first but they become closer and closer as $\lambda$ grows. As can be seen, the classification error becomes almost zero when $\lambda \approx 50$, but then it starts going up until it reaches the moment when $W = 0$ and therefore the error is equal to $\frac{k-1}{k} = 0.8$. The second plot depicts the predicted sparsity rate. As can be seen, we achieve $\approx 15X$ sparsification at $\lambda \approx 50$, where the error is optimal and, in fact, very close to that obtained from ridge regression. Thus, we obtain 15X compression with minimal loss of performance. Finally, note that this $\lambda$ turns out to be large enough for the true error, the error of the sparsified solution, and the predicted error to match very closely. 

\begin{figure}[htp]
    \centering
    \includegraphics[width = 3.25 in]{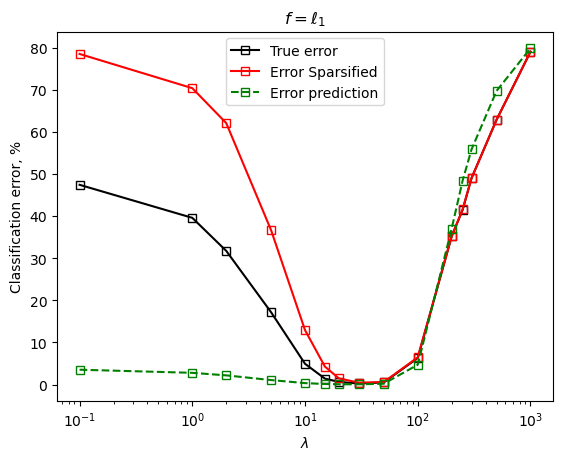}
    \includegraphics[width = 3.25 in]{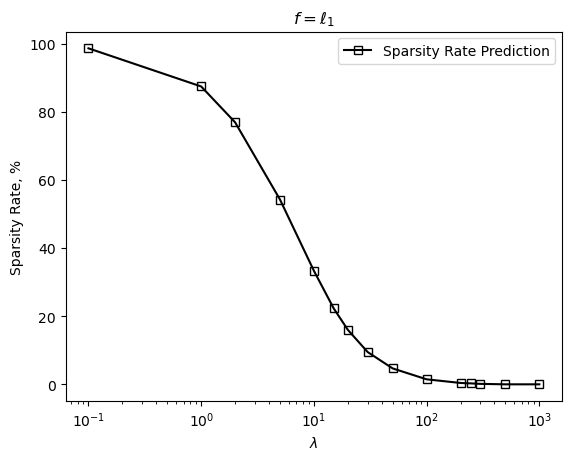}
    \caption{We took $d = 750$, $n = 500$, $k = 5$, $r = 0.8$, $c =0.3$ and $\sigma = 1$ for these plots. They illustrate that for these parameters it is possible to sparsify the weights by 15X while keeping the classification error very low.  }
    \label{fig:l1_750}
\end{figure}

\subsection{$f(\cdot) = \|\cdot\|_\infty$}

We showcase the results of our experiments for $f(\cdot) = \|\cdot\|_\infty$ in Fig.\ref{fig:linf_750}. The plot on the top contains three lines: the true classification error of the solution $W$ we obtained by solving (\ref{eq:est}) directly, the classification error for the compressed version of $W$, which we simply took to be $sign(W)$ due to the scaling invariance of the error function and permutation symmetry of the model, and the classification error prediction from Theorem \ref{thm: linf}. As expected, they differ at first, but start matching very closely for large values of $\lambda$. The plot on the bottom shows the percentage of weights on the boundary for each $\lambda$. As can be seen, it is almost $1$ at the end. Finally, unlike the $f(\cdot) = \|\cdot\|_1$ case, we do not observe a drop of performance after a certain $\lambda$, so the classification error for $f(\cdot) = \|\cdot\|_\infty$ behaves much more similarly to ridge regression. It should be noted, however, that this is only an empirical observation at this point. Analytically proving whether this is generally true or not is left for future work. 

\begin{figure}[htp]
    \centering
    \includegraphics[width = 3.25 in]{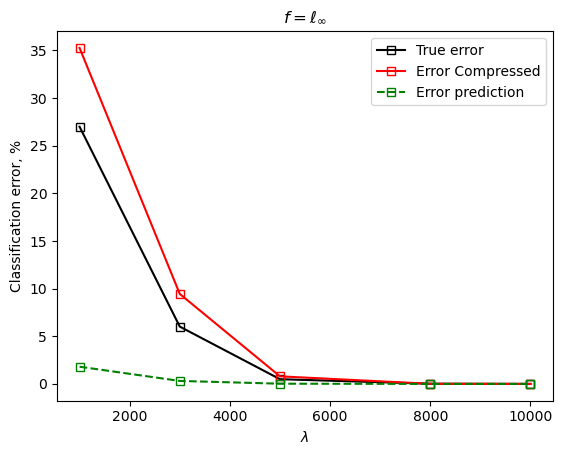}
    \includegraphics[width = 3.25 in]{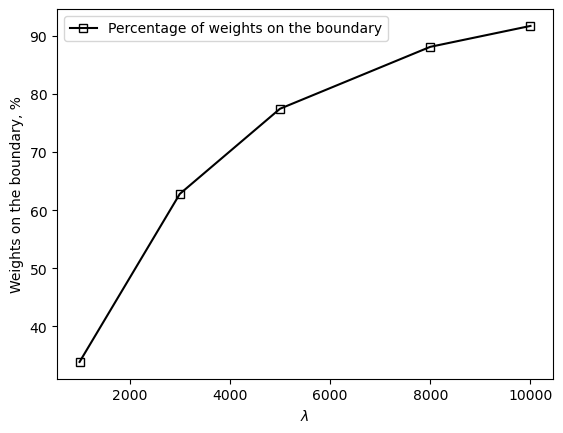}
    \caption{We took $d = 750$, $n = 500$, $k = 5$, $r = 0.8$, $c =0.3$ and $\sigma = 1$ for these plots. They illustrate that for these parameters it is possible to compress each weight to one bit while keeping the classification error very low.  }
    \label{fig:linf_750}
\end{figure}

\section{Conclusion}
We have rigorously analyzed the performance of regularized linear regression for multi-class classification in the presence of labeling errors in the strong regularization regime. Theory and simulations suggest that the optimal classification error is found by ridge regression with large regularization parameter as well as that using $f(\cdot) = \|\cdot\|_\infty$ and $f(\cdot) = \|\cdot\|_1$ and regularization one can obtain one-bit and sparse solutions, respectively, with negligible loss of performance. Studying the implications of these results, and whether they generalize, to nonlinear regression models, such as deep neural networks, is an interesting topic for future study.

\newpage


\bibliographystyle{apalike}
\bibliography{main}

\newpage
\appendix
\onecolumn

\section{Proof of Equation \ref{eq: class_error}}
We begin with presenting the classification error expression in a more tractable format. After dropping the all-zero column and row from $S_{\ell}$, we have for $i,j \neq \ell$:
\begin{align*}
    (S_{\ell})_{ij} = (w_{\ell} - w_i)^T (w_{\ell} - w_j) =  \|w_{\ell}\|^2 - w_{\ell}^T (w_i + w_j) + w_i^T w_j
\end{align*}
Now note that, due to symmetry of our model, $w_{\ell}^T w_i = w_{\ell}^T w_j$. Also $\|w_{\ell}\|^2 = \|w_i\|^2$. Therefore 
\begin{align*}
    (S_{\ell})_{ij} = \begin{cases}
        2 \|w_{\ell}\|^2 - 2 w_{\ell}^T w_i & i = j \\
         \|w_{\ell}\|^2 - w_{\ell}^T w_i & i \neq j
    \end{cases} = \frac{1}{2} \|w_{\ell} - w_i \|^2 \begin{cases}
        2 & i = j \\
        1 & i \neq j
    \end{cases} 
\end{align*}
Hence $S_{\ell} = \frac{1}{2} \|w_{\ell} - w_i \|^2 (\mathds{1}\mathds{1}^T  + I)$. Subsequently, $S_\ell^{1/2} = \frac{1}{\sqrt{2}} \|w_{\ell} - w_i \| (\sqrt{k}\frac{\mathds{1}\mathds{1}^T}{k-1} + I - \frac{\mathds{1}\mathds{1}^T}{k-1}) = \frac{1}{\sqrt{2}} \|w_{\ell} - w_i \| (I + \frac{\mathds{1}\mathds{1}^T}{1+\sqrt{k}})$ Hitting it with $g\in \bbR^{k-1}$, $g \sim \calN(0, \sigma^2 I)$. Thus the classification error would be
\begin{align*}
     P_{e|\ell} =  1- \bbP\left((I + \frac{\mathds{1}\mathds{1}^T}{1+\sqrt{k}}) g \ge \sqrt{2}\mu_{\ell}^T \frac{w_i-w_{\ell}}{\|w_i-w_{\ell}\|} \mathds{1}\right)
\end{align*}

\section{Proof of Lemma \ref{lem: tran_PO}}
It will be instructful to replace $X$ by (\ref{eq:X}) and plug it in into (\ref{eq:est}) to obtain
\begin{equation*}\sum_{\ell=1}^k\min_{w_\ell}\left[ \|Aw_\ell+ \Ytil Mw_\ell-Y_\ell\|_2^2+\lambda f(w_\ell)\right] 
\end{equation*}
Since it suffices to find the pairwise interactions between pairs of $w_\ell, w_{\ell'}$, we will look at the terms containing $w_\ell$ and $w_{\ell'}$ only. This gives us:
\begin{align*} 
 \min_{w_\ell, w_{\ell'}} \|Aw_\ell+ \Ytil Mw_\ell-Y_\ell\|_2^2+ \|Aw_{\ell'}+ \Ytil Mw_{\ell'}-Y_{\ell'}\|_2^2  + \lambda f(w_\ell) +  \lambda f(w_{\ell'}) 
\end{align*}

We would now like to find an orthogonal operator $U \in \mathbb{R}^{n \times n}$, $\ell = 1, \dots, k$ to apply to the equation above that would help us simplify it via
\begin{align}
\min_{w_\ell, w_{\ell'}} \|UAw_\ell+ U\Ytil Mw_\ell-UY_\ell\|_2^2  + \|UAw_{\ell'}+ U\Ytil Mw_{\ell'}-UY_{\ell'}\|_2^2+ \lambda f(w_\ell) +  \lambda f(w_{\ell'})
\label{eq:u_l}
\end{align}
We will require $U$ to satisfy $U\Ytil = \left[ \begin{array}{cccc} \sqrt{\frac{n}{k}}I_k & 0 & \ldots & 0 \end{array} \right]^T$. The full set of constraints it imposes on $UY_\ell $  can be described as $$\sqrt{\frac{n}{k}}(UY_\ell)_{\ell'} = (UY_\ell)^T(U\Ytil_{\ell'}) = Y_\ell^T{\Ytil_{\ell'}} = \begin{cases}
    \frac{n}{k}(1-c) &\text{ if } \ell' = \ell \\
    \frac{nc}{k(k-1)} &\text{ if } \ell' \ne \ell \\
\end{cases} \text{ for } \ell, \ell' = 1, \dots, k ~~~\mbox{and}~~~ \|UY_\ell\|^2 = \|Y_\ell\|^2 = \frac{n}{k}$$
Put together, these let us require the following equalities, modulo verification that the the desired $s, t \in \mathbb{R}$ exist:
$$(UY_\ell )_i = \begin{cases}
    \sqrt{\frac{n}{k}}(1-c) &\text{ if } i = \ell \\
    \sqrt{\frac{n}{k}}\frac{c}{k-1} &\text{ if } i \ne \ell, i \le k \\
    \sqrt{\frac{n}{k}}s & \text{ if } i = k + 1 \\
    \sqrt{\frac{n}{k}}t & \text{ if } i = k + 2 \\
    0 & \text{ if } i > k+2     
\end{cases} ~~~\mbox{and}~~~ (UY_{\ell'} )_i = \begin{cases}
    \sqrt{\frac{n}{k}}(1-c) &\text{ if } i = \ell' \\
    \sqrt{\frac{n}{k}}\frac{c}{k-1} &\text{ if } i \ne \ell', i \le k \\
    \sqrt{\frac{n}{k}}t & \text{ if } i = k + 1 \\
    \sqrt{\frac{n}{k}}s & \text{ if } i = k + 2 \\
    0 & \text{ if } i > k+2     
\end{cases} $$
Here $s,t$ are determined by the following system of equations:  
$$\frac{n}{k}\left[(1-c)^2+(k-1)(\frac{c}{k-1})^2+s^2+t^2 \right] = \|Y_\ell\|^2 = \frac{n}{k}~~~\mbox{and}~~~ \frac{n}{k}\left[\frac{2c(1-c)}{k-1}+(k-2)(\frac{c}{k-1})^2+2st\right] = Y_{\ell}^TY_{\ell'} = 0$$
Simplifying both we have: 
$$s^2+t^2  = 2c-\frac{kc^2}{k-1}~~~\mbox{and}~~~ 2st = -\frac{2c-\frac{kc^2}{k-1}}{k-1}$$
Since $c < \frac{1}{2}$ and $k \ge 2$, we have $2c-\frac{kc^2}{k-1} \ge c  \ge 0$. Hence, $s$ and $t$ are well-defined and can be identified via: 
$$s = \frac{1}{2}\left[\sqrt{\frac{(k-2)(2c-\frac{kc^2}{k-1})}{k-1}} + \sqrt{\frac{k(2c-\frac{kc^2}{k-1})}{k-1}} \right] ~~~\mbox{and}~~~ t = \frac{1}{2}\left[\sqrt{\frac{(k-2)(2c-\frac{kc^2}{k-1})}{k-1}} - \sqrt{\frac{k(2c-\frac{kc^2}{k-1})}{k-1}} \right]$$

Note that (\ref{eq:u_l}) now transforms into the following, where $A_i$ stands for the $i-th$ row of $A$, $\mutil_i^T = \mu_i^T + \sqrt{\frac{k}{n}}A_i$ for $i = 1 ,\dots, k$, $a = A_{k+1}, b = A_{k+2}$ and $\tilde{A} \in \mathbb{R}^{(n-k-2)\times d}$ is obtained from $A$ by throwing the first $k+2$ rows away:

\begin{align}
& \min_{w_{\ell}, w_{\ell'}}  \|\tilde{A}w_{\ell}\|^2 + \frac{n}{k}\left[ \sum_{\msr \ne \ell' }(w_{\ell}^T\mutil_{\msr}-\frac{c}{k-1})^2+(w_{\ell}^T\mutil_{\ell}-(1-c))^2+(\sqrt{\frac{k}{n}}w_{\ell}^Ta-s)^2+(\sqrt{\frac{k}{n}}w_{\ell}^Tb-t)^2 \right]+ \lambda f(w_{\ell})  \\ &+\|\tilde{A}w_{\ell'}\|^2 + \frac{n}{k}\left[ \sum_{\msr \ne \ell' }(w_{\ell'}^T\mutil_{\msr}-\frac{c}{k-1})^2+(w_{\ell'}^T\mutil_{\ell'}-(1-c))^2+(\sqrt{\frac{k}{n}}w_{\ell'}^Ta-t)^2+(\sqrt{\frac{k}{n}}w_{\ell'}^Tb-s)^2 \right]+ \lambda f(w_{\ell'})
\label{eq: 2_term_po}
\end{align}

\section{Sandwiching Lemmas} \label{sec: sandwich}
To provide insights on why we consider the lower and upper bound optimization, we utilize CGMT. To do so, we analyze the terms containing $w_\ell$ and $w_{\ell'}$ separately first. Without loss of generality we will consider only the one with $w_{\ell}$. Starting from (\ref{eq: POl}) for $\Phi_{\ell,s,t}(w_\ell)$, applying Fenchel dual to $\|\tilde{A}w_\ell\|^2$, we obtain:
\begin{equation*}\min_{w_\ell} \max_{u_\ell} u^T_\ell\tilde{A}w_\ell - \frac{\|u_\ell\|^2}{4} + \frac{n}{k}\sum_{\msr \ne \ell }(w_\ell^T\mutil_{\msr}-\frac{c}{k-1})^2+\frac{n}{k}(w_\ell^T\mutil_{\ell}-(1-c))^2+\frac{n}{k}(\sqrt{\frac{k}{n}}w_\ell^Ta-s)^2+\frac{n}{k}(\sqrt{\frac{k}{n}}w_\ell^Tb-t)^2 + \lambda f(w_\ell)
\end{equation*}
Using CGMT for the objective above:
\begin{align*}
& \min_{w_\ell} \max_{u_\ell}w_\ell^Tg^{(\ell)}\|u_\ell\| + u_\ell^Th^{(\ell)}\|w_\ell\| - \frac{\|u_\ell\|^2}{4} + \frac{n}{k}\sum_{\msr \ne \ell }(w_\ell^T\mutil_{\msr}-\frac{c}{k-1})^2+\frac{n}{k}(w_\ell^T\mutil_{\ell}-(1-c))^2 + \\
& +\frac{n}{k}(\sqrt{\frac{k}{n}}w_\ell^Ta-s)^2+\frac{n}{k}(\sqrt{\frac{k}{n}}w_\ell^Tb-t)^2 + \lambda f(w_\ell) \text{, where } g^{(\ell)} \sim \calN(0, \sigma^2I_d) \text{ and }h^{(\ell)} \sim \calN(0, \sigma^2I_n)
\end{align*}

It is straightforward to see that for the optimal $u_\ell$  one has $u_\ell = \frac{\eta_\ell}{\|h^{(\ell)}\|} h^{(\ell)}$, where $\eta_\ell = \|u_\ell\|$ and the expression above turns into:
\begin{align*}
& \min_{w_\ell} \max_{\eta_\ell \ge 0}w_\ell^Tg^{(\ell)}\eta_\ell + \eta_\ell \sigma\sqrt{n} \|w_\ell\| - \frac{\eta_\ell^2}{4} + \frac{n}{k}\sum_{\msr \ne \ell }(w_\ell^T\mutil_{\msr}-\frac{c}{k-1})^2+\frac{n}{k}(w_\ell^T\mutil_{\ell}-(1-c))^2 + \\
& +\frac{n}{k}(\sqrt{\frac{k}{n}}w_\ell^Ta-s)^2+\frac{n}{k}(\sqrt{\frac{k}{n}}w_\ell^Tb-t)^2 + \lambda f(w_\ell)
\end{align*}
Optimizing over $\eta_\ell$ we get:
\begin{align*}\min_{w_\ell} (w_\ell^Tg^{(\ell)} + \sigma\sqrt{n} \|w_\ell\|)_{\ge 0}^2 + \frac{n}{k}\left[\sum_{\msr \ne \ell }(w_\ell^T\mutil_{\msr}-\frac{c}{k-1})^2+(w_\ell^T\mutil_{\ell}-(1-c))^2 +(\sqrt{\frac{k}{n}}w_\ell^Ta-s)^2+(\sqrt{\frac{k}{n}}w_\ell^Tb-t)^2 \right]+ \lambda f(w_\ell)
\end{align*}
In the above optimization, one can decompose the optimal solution into $w^*_\ell = \alpha g_\ell + w_\ell^\perp$ with $g^T_\ell w_\ell^\perp = 0 $ for some $\alpha \in \bbR$. now, $w^{T*}_\ell g = \alpha \|g\|^2$, and in the objective, $(w_\ell^Tg^{(\ell)} + \sigma\sqrt{n} \|w_\ell\|)_{\ge 0}^2 = (\alpha \|g\|^2 + \sigma\sqrt{n} \|w_\ell\|)_{\ge 0}^2 $. Note that the optimal $w^*_{\ell}$, to minimize the cost would always have $\alpha \le 0$, which shows $(w_\ell^Tg^{(\ell)} + \sigma\sqrt{n} \|w_\ell\|)_{\ge 0}^2 \le n \sigma^2 \|w_{\ell}^*\|^2$. 
To formalize this idea, we introduce the following notations for the subsequent lemmas:
$$ F(w_\ell) =  \frac{n}{k}\left[\sum_{\msr \ne \ell }(w_\ell^T\mutil_{\msr}-\frac{c}{k-1})^2+(w_\ell^T\mutil_{\ell}-(1-c))^2 +(\sqrt{\frac{k}{n}}w_\ell^Ta-s)^2+(\sqrt{\frac{k}{n}}w_\ell^Tb-t)^2 \right]$$
$$\phil (w_\ell):= F(w_\ell) + \lambda f(w_\ell)$$
$$
\phi_\lambda (w_\ell) := (w_\ell^Tg^{(\ell)} + \sigma\sqrt{n} \|w_\ell\|)_{\ge 0}^2 +  F(w_\ell) + \lambda f(w_\ell)
$$
$$\phiu (w_\ell):= \sigma^2n\|w\|^2 + F(w_\ell) + \lambda f(w_\ell)$$
Lemma \ref{lem: obj sand} shows the inequality relation among the $\phi_\lambda, \phil, \phiu$ on which the proof of Theorem \ref{thm : solutions_match} relies heavily. Lemmas \ref{lem: equality} and \ref{lem: exists} complete the sandwiching argument by showing that the limit of the lower bound and the upper bound match.

\begin{lemma} \label{lem: obj sand}
The following inequalities hold for any $\lambda$
    $$\min_w \phil(w) \leq \min_w \phi(w) \leq \min_w \phiu(w)$$
\end{lemma}

\begin{proof}
    The inequality on the left-hand side holds because $\phil(w) \leq \phi(w)$ is satisfied for every $w$. To justify the right hand side one, take $w_* = \arg min_w \phiu(w)$. Since $g$ is random Gaussian and $w_*$ is independent from it, $w_*^Tg = 0$ and $\min_w \phi(w) \leq \phi(w_*) = \phiu(w_*)$
\end{proof}
\begin{lemma} \label{lem: exists}
    The following limit $\displaystyle \lim_{\lambda \to \infty} \min_w \phiu(w)$ exists, is finite and non-zero
\end{lemma}
\begin{proof}
    It suffices to show that $\Theta(\lambda) := \min_w \phiu(w)$ is bounded from above and increases in $\lambda$. To show the former, note that $\Theta(\lambda) \leq \phiu(0) = \frac{c^2}{k-1} + (1-c) ^ 2 + s^2 + t^2$.  To show the latter, note that $\phi_{\lambda}$ is linear so by appealing to Danskin's theorem, we take the derivative at the optimal point and observe that it is non-negative: $\Theta'(\lambda) = f(w_*) \geq  0$
\end{proof}

\begin{lemma}  \label{lem: equality}
Assume that the regularizer $f(w)$ satisfies $\|w\|_2 \le q(d)f(w)$ for some $q: \mathbb{N} \to \mathbb{R}$. Then the following equality holds:
    $$\lim_{\lambda \to \infty} \min_w \phil(w) = \lim_{\lambda \to \infty} \min_w \phiu(w)$$
\end{lemma}

\begin{proof}
    Note that $\displaystyle \lim_{\lambda \to \infty} \min_w \phil(w) \leq \lim_{\lambda \to \infty} \min_w \phiu(w)$ is implied by the previous lemma. To prove the converse inequality, note that for the purposes of minimization  we can restrict both $\phil$ and $\phiu$ to the unit ball for large enough $\lambda$. Indeed, suppose the contrary and assume $\|w_*\| \geq 1$ can happen for arbitrary large $\lambda$, we then have $\lambda f(w) \ge \lambda q(d) \|w\| \ge \lambda q(d)$, which is not bounded as a function of $\lambda$, giving us a contradiction. Also taking $\lambda \geq \sigma^2n q(d)$, we write 
    $$\lim_{\lambda \to \infty} \min_w \phiu(w) = \lim_{\lambda \to \infty} \min_w F(w) + \lambda (\frac{\sigma^2n \|w\|^2}{\lambda} + f(w)) \leq \lim_{\lambda \to \infty} \min_w F(w) + \lambda (\frac{\sigma^2n \|w\|}{\lambda} + f(w)) \leq$$ 
    $$ \leq \lim_{\lambda \to \infty} \min_w F(w) + \lambda (\frac{\sigma^2n q(d)}{\lambda} + 1)f(w) \leq \lim_{\lambda \to \infty} \min_w F(w) + 2\lambda f(w) = \lim_{\lambda \to \infty} \min_w \phil(w) $$
\end{proof}
\section{Proof of Theorem \ref{thm : solutions_match}}
We first state the following key lemma which will be referred to multiple times in this section. We also suppress notation by setting $\Phi_{\ell} := \Phi_{\ell, s, t}$
\begin{lemma} \label{lem: stg cvx}
    Let $f(x,y):\bbR^d  \times \bbR^d \rightarrow \bbR$ be a convex function on a compact set $U$. Take a $L$-Lipschitz function $\psi:  \bbR^d \rightarrow \bbR $.   For $\epsilon > 0$, consider $f^\epsilon(x,y) := \epsilon \|x\|^2 + \epsilon \|y\|^2 + f(x,y)$. Let $(x^*, y^*)$ be a minimum point of $f^\epsilon$. Then the following inequality holds:
    \begin{align*}
        f^\epsilon(x,y) \ge \min_{w,v} f^\epsilon(w,v) + \frac{\epsilon}{L^2}(\psi(x)-\psi(x^*))^2 + \frac{\epsilon}{L^2}(\psi(y)-\psi(y^*))^2 
    \end{align*}
\end{lemma}
\begin{proof}
    We proceed by the same approach as Lemma 11 in \cite{panahi2017universal}. Since $(x^*, y^*)$ is a minimum point of $f^\epsilon$, then $2 \epsilon x^* \in -\partial_x f(x^*, y^*)$ and $2 \epsilon y^* \in -\partial_y f(x^*, y^*)$ , then by convexity:
    \begin{align*}
        f(x,y) - f(x^*, y^*) \ge  -2\epsilon x^{T*} (x - x^*) - 2\epsilon y^{T*} (y - y^*)
    \end{align*}
    Therefore
    \begin{align*}
        &f^\epsilon (x,y) - f^\epsilon(x^*, y^*) = \epsilon(\|x\|^2 - \|x^*\|^2) +  \epsilon(\|y\|^2 - \|y^*\|^2) + f(x,y) - f(x^*, y^*) \ge \\
        &\epsilon(\|x\|^2 - \|x^*\|^2) +  \epsilon(\|y\|^2 - \|y^*\|^2) -2\epsilon x^{T*} (x - x^*) - 2\epsilon y^{T*} (y - y^*) = \epsilon ( \|x\|^2 +  \|x^*\|^2 - 2 x^T x^*) + \epsilon ( \|y\|^2 +  \|y^*\|^2 - 2 y^T y^*)
    \end{align*}
    Hence 
    \begin{align*}
        f^\epsilon (x,y) \ge f^\epsilon(x^*, y^*) + \epsilon \|x-x^*\|^2 + \epsilon \|y-y^*\|^2
    \end{align*}
    Now by Lipschitzness of $\psi$, we have:
    \begin{align*}
        f^\epsilon (x,y) \ge f^\epsilon(x^*, y^*) + \frac{\epsilon}{L^2} (\psi(x) - \psi(x^*))^2 + \frac{\epsilon}{L^2} (\psi(y) - \psi(y^*))^2
    \end{align*}
    Which concludes the proof
\end{proof}
For ease of proof, we require the objectives in the lower bound optimization to be strongly convex. We achieve this by adding $\epsilon \|w\|^2$ to the objective $\min_{w_\ell} \Phi_\ell(w)$ for a small $\epsilon > 0$. It has been proved in Theorem 2 in \cite{panahi2017universal} that for the LASSO regression, the solutions do not change much as $\epsilon$ is chosen to be small. Here we prove this statement for the large $\lambda$ regime for the regularizer satisfying Assumption (A4). Let us start by defining few notations. Let
\begin{align*}
    &\Phi_{\ell}^\epsilon (w) := \Phi_{\ell} (w) + \epsilon \|w\|^2 \\
    & w_{\Phi_\lambda}:= argmin (\Phi_{\ell} (w)) , \quad 
    w_{\Phi^\epsilon_\lambda}:= argmin (\Phi^\epsilon_{\ell} (w))
\end{align*}
\begin{lemma}
     Let $\eta >0$ be given then for a large enough $\lambda$ and any $L$-Lipschitz function, If $\bbP(| \psi(w_{\Phi_\lambda}^\epsilon)- \alpha_*| > \eta) \rightarrow 0$, then  $\bbP(| \psi(w_{\Phi_\lambda})- \alpha_*| > \eta) \rightarrow 0$
\end{lemma}
\begin{proof}
Let $\lambda$ be large enough such that $\|w_{\Phi_\lambda}\|\le1$ which is proven in lemma \ref{lem: equality}. Then by assumption (A4), there exists $\lambda'$ such that $\epsilon \|w\|^2 + \lambda f(w) \le \lambda' f(w)$ for all $w$ that $\|w\|\le M < \infty$. This follows as we only perform the optimization over the compact set $\mathcal{S}_w$. Indeed, taking any $\lambda'\ge \lambda + \frac{M\epsilon}{q(d)}$, guarantees
\begin{align*}
    (\lambda' - \lambda) f(w) \ge \frac{M\epsilon}{q(d)} f(w) \ge M \epsilon  \|w\| \ge \epsilon \|w\|^2
\end{align*}
This implies for $w\in \mathcal{S}_w$
\begin{align*}
     \Phi_{\lambda'}(w) \ge \Phi_{\lambda}^{\epsilon}(w)
\end{align*}
Now, we adopt a similar approach to \cite{thrampoulidis2015regularized} and let
\begin{align*}
    J_{\eta} := \{\alpha | |\alpha - \alpha_*| > \eta\}
\end{align*} 
Then 
\begin{align*}
    \min_{\psi(w) \in J_\eta} \Phi_{\lambda'} \ge \min_{\psi(w) \in J_\eta} \Phi_{\lambda}^\epsilon 
\end{align*}
By lemma \ref{lem: stg cvx}, 
\begin{align*}
    \Phi_{\lambda}^\epsilon \ge \min_v \Phi_{\lambda}^\epsilon(v) + \frac{\epsilon}{L^2}(\psi(w)-\psi(w_{\Phi_\lambda}))^2
\end{align*}
Therefore,
\begin{align*}
    \min_{\psi(w) \in J_\eta} \Phi_{\lambda}^\epsilon \ge \min_v \Phi_{\lambda}^\epsilon(v) + \frac{\epsilon \eta^2}{L^2}
\end{align*}
By applying the same reasoning as lemma \ref{lem: equality}, we observe
\begin{align*}
\lim_{\lambda \to \infty} \min_w \Phi^\epsilon_\lambda(w) = \lim_{\lambda \to \infty} \min_w \phi_\lambda(w)    
\end{align*}
Now by taking $\lambda$ large enough we obtain $\min_v \Phi_{\lambda'}^\epsilon(v) \le \min_v \Phi_{\lambda}^\epsilon(v) + \epsilon'$ with $\epsilon' < \frac{\epsilon \eta^2}{L^2}$. Therefore
\begin{align*}
    \min_{\psi(w) \in J_\eta} \Phi_{\lambda'} > \min \Phi_{\lambda' }^\epsilon \ge \min \Phi_{\lambda}^\epsilon \ge \min \Phi_{\lambda'}
\end{align*}
As $\lambda'\ge \lambda + \frac{M\epsilon}{q(d)}$ is arbitrary large, this shows that with high probability, $|\psi(w_{\Phi_\lambda})-\alpha_*|<\eta$.
\end{proof}
Let us define $w_- := w_{\ell}-w_{\ell'}$ and $w_+ := w_{\ell}+w_{\ell'}$. Under this change of variable, note that the $F$ in the eq. (\ref{eq : phi}) would become
    \begin{align*}&\tilde{F}(w_+, w_-):= \frac{n}{2k}(\sum_{\msr \ne \ell, \ell' }[(w_-^T\mutil_{\msr})^2 + (w_+^T\mutil_{\msr}-\frac{2c}{k-1})^2] +(w_+^T\mutil_{\ell}-(1-c+\frac{c}{k-1}))^2 + (w_-^T\mutil_{\ell}-(1-c-\frac{c}{k-1}))^2 + \\
    & (w_+^T\mutil_{\ell'}-(1-c+\frac{c}{k-1}))^2 + (w_-^T\mutil_{\ell}-(\frac{c}{k-1}-1+c))^2 + (\sqrt{\frac{k}{n}}w_+^Ta-s-t)^2 + \\ &(\sqrt{\frac{k}{n}}w_+^Ta-s+t)^2+(\sqrt{\frac{k}{n}}w_+^Tb-s-t)^2 ) + (\sqrt{\frac{k}{n}}w_+^Tb-t+s)^2 )
\end{align*}
Define the following transformed optimizations:
\begin{align*}
    &\tilde{\Phi}^{\epsilon}(w_+,w_-) := \epsilon\|w_+\|^2 + \epsilon \|w_-\|^2 + \frac{1}{2}\|\tilde{A}w_+\|^2 + \frac{1}{2}\|\tilde{A}w_-\|^2 + \tilde{F}(w_+, w_-) + \lambda f(\frac{w_++w_-}{2}) + \lambda f(\frac{w_+-w_-}{2}) \\
    &\phil^\epsilon(w_+, w_-) := \epsilon \|w_-\|^2 + \epsilon \|w_+\|^2+ \tilde{F}(w_+, w_-) + \lambda f(\frac{w_++w_-}{2}) + \lambda f(\frac{w_+-w_-}{2}) \\
    &\phiu^\epsilon(w_+, w_-) := (\epsilon + n\sigma^2) \|w_-\|^2 + (\epsilon+n\sigma^2) \|w_+\|^2+ \tilde{F}(w_+, w_-) + \lambda f(\frac{w_++w_-}{2}) + \lambda f(\frac{w_+-w_-}{2})
\end{align*}
Further we denote
\begin{align*}
    (\hat{w}_-, \hat{w}_+) := argmin(\tilde{\Phi}^{\epsilon}) \\
    (\wl_-, \wl_+):= argmin(\phil^\epsilon) \\
    (\wu_-, \wu_+):= argmin(\phiu^\epsilon) 
\end{align*}
The following Lemma is key in the proof of Theorem \ref{thm : solutions_match}
\begin{lemma}
    Let $\eta >0$ be given, then for a large enough $\lambda$ and any $L$-Lipschitz function, If $\bbP(| \psi(\wl_-)- \alpha_* | > \eta) \rightarrow 0$, then  $\bbP(| \psi(\hat{w}_-)- \alpha_* | > \eta) \rightarrow 0$ and $\bbP(| \psi (\wu)- \alpha_* | > \eta) \rightarrow 0$ as $n\rightarrow 0$.
    \label{lem: large_lam_approx}
\end{lemma}

\begin{proof}
As $\tilde{\Phi}^{\epsilon}$ is always greater than $\phil ^{\epsilon}$ on every point in its domain the following inequality holds w.p.a 1:
\begin{align*}
    \min_{\psi(w_-) \in J_{\eta}} \tilde{\Phi}^{\epsilon} \ge  \min_{\psi(w_-) \in J_{\eta}} \phil^\epsilon
\end{align*}

By the previous lemma, the following inequality holds with probability 1.
\begin{align*}
    \phil^{\epsilon}(w_+, w_-) \ge \min_{w_-, w_+} \phil^{\epsilon} + \frac{\epsilon}{L^2} (\psi(w_-) - \psi(\wl_-))^2 + \frac{\epsilon}{L^2} (\psi(w_+) - \psi(\wl_+))^2) \ge  \min_{w_-, w_+}\phil^{\epsilon} + \frac{\epsilon}{L^2} (\psi(w_-) - \psi(\wl_-))^2
\end{align*}
Thus minimizing over the set $J_\eta$ results in
\begin{align*}
     \min_{\psi(w_-) \in J_{\eta}} \phil^{\epsilon}(w_+, w_-) \ge \min_{w_-, w_+} \phil^{\epsilon} + \frac{\epsilon}{L^2} \eta^2
\end{align*}
By Lemma \ref{lem: equality} for large enough lambda, there exists $\tilde{\epsilon}$ such that
\begin{align*}
    \min \phiu^{\epsilon}\le \min \phil^{\epsilon} + \tilde{\epsilon}
\end{align*}
We can take $ \tilde{\epsilon} <\frac{\epsilon}{L^2} \eta^2$, wp 1. This implies
\begin{align}\label{ineq: match}
     \min_{\psi(w_-) \in J_{\eta}} \phil^{\epsilon}(w_+, w_-) \ge \min_{w_-, w_+} \phil^{\epsilon} + \frac{\tau}{2} \eta^2 > \min_{w_-, w_+} \phiu^{\epsilon} 
\end{align}
We also know that $\min_{w_-, w_+} \phiu^{\epsilon} = \min_{w_1,w_2} \phiu^{\epsilon}$, which implies the following bound 
\begin{align*}
    \min_{w_-, w_+} \phiu^{\epsilon} \ge \min \tilde{\Phi}^{\epsilon}
\end{align*}
All in all, we have w.p.a 1
\begin{align*}
    \min_{\psi(w_-) \in J_{\eta}} \tilde{\Phi}^{\epsilon} > \min_{w_-,w_+} \tilde{\Phi}^{\epsilon}
\end{align*}
This implies the solution $\psi(\hat{w}_-) \notin J_{\eta}$ w.p.a 1 and thus $\bbP(| \psi(\hat{w}_-)- \alpha_* | > \eta) \rightarrow 0$.
Now for the $\psi(\wu)$, the following holds w.p.a 1:
\begin{align*}
    \min_{\psi(w_-) \in J_{\eta}} \phiu^{\epsilon} \ge  \min_{\psi(w_-) \in J_{\eta}} \phil^\epsilon
\end{align*}
As $\phiu^{\epsilon}$ is always greater than $\phil ^{\epsilon}$ on every point in its domain. Then one has from the inequality (\ref{ineq: match})
\begin{align*}
     \min_{\psi(w_-) \in J_{\eta}} \phil^{\epsilon}(w_+, w_-) \ge \min_{w_-, w_+} \phil^{\epsilon} + \frac{\tau}{2} \eta^2 > \min_{w_-, w_+} \phiu^{\epsilon} 
\end{align*}
To summarize,
\begin{align*}
    \min_{\psi(w_-) \in J_{\eta}} \phiu^{\epsilon} > \min_{w_-, w_+} \phiu^{\epsilon} 
\end{align*}
Which implies $\psi(\wu) \notin J_\eta$ w.p.a 1 and $\bbP(| \psi (\wu)- \alpha_* | > \eta) \rightarrow 0$. This concludes the proof
\end{proof}
Finally, using the previous results we can prove Theorem \ref{thm : solutions_match}:
\begin{proof}(Proof of Theorem \ref{thm : solutions_match})
By equation \ref{eq: class_error} we know that the classification error is characterized by $ \mu_\ell^T (w^*_{\ell'}-w^*_{\ell})$ and $\|w^*_{\ell'}-w^*_{\ell})\|$. Therefore, the proof of Theorem \ref{thm : solutions_match} follows by applying Lemma \ref{lem: large_lam_approx} twice to the Lipschitz functions  $\psi_1(x) := \|x\|$ and $\psi_2(x) := \mu_1^T x$.
\end{proof}
Lemma \ref{lem: large_lam_approx} has further implications such as the following Corollary, for the proof of which we refer to the Theorem IV.1 in \cite{thrampoulidis2018symbol}.     

\begin{corollary}
    \label{cor: dist_match}
    If the distribution of $\wl$ converges to $p$, then so do the distributions of $\wu$ and $\hat{w}$.
\end{corollary}

\section{Proof of Theorem \ref{thm: l2}}
\begin{theorem}
    Assume that (A1) and (A2) hold. Then the classification error corresponding to $f(\cdot) = \|\cdot\|_2^2$ can be described as follows in the large $\lambda$ regime with probability approaching $1$ as $n$ grows:
    $$\lim_{\lambda \to \infty}P_{e, \lambda} = \lim_{\lambda \to \infty} Q_k(\frac{d(\zeta - \gamma)(1 - r)}{2d(\zeta - \gamma)^2(1 + \frac{\sigma^2k}{n} - r) + 2(\alpha - \beta) ^ 2\sigma ^ 2d})$$
    where
\begin{align*}
    \tilde{\lambda} & = \frac{\lambda k}{n} + \sigma^2k  \\
     \Delta & = \left[((k-2)r+1+\frac{k\sigma^2}{n})d + \tilde{\lambda}\right](d(1+\frac{k\sigma^2}{n}) + \tilde{\lambda}) - (k-1)(rd)^2 \\
    \gamma & =  \frac{1}{\Delta}\left[(d(1+\frac{k\sigma^2}{n}) + \tilde{\lambda})\frac{c}{k-1} -rd(1-c)  \right]\\
    \zeta & = \frac{1}{\Delta}\left[-crd + (1-c)(((k-2)r+1+\frac{k\sigma^2}{n})d + \tilde{\lambda}) \right] \\
    \alpha & = \frac{s\sqrt{nk}}{n\tilde{\lambda}+\sigma^2dk} \\
    \beta & = \frac{t\sqrt{nk}}{n\tilde{\lambda}+\sigma^2dk}
\end{align*}
\label{thm: l2_complete}
\end{theorem}

\begin{proof}
    
According to Theorem \ref{thm : solutions_match}, it suffices to solve $\phiu^{\ell}(w_\ell) + \phiu^{\ell'}(w_{\ell'})$ to evaluate the classification error in the large lambda regime. Setting the gradients of $\phiu^{\ell}(w_\ell) + \phiu^{\ell'}(w_{\ell'})$ by $w_\ell$ and $w_{\ell'}$ to $0$, we obtain: 
\begin{align} & \frac{n}{k}\left[\sum_{\msr \ne \ell }(w_\ell^T\mutil_{\msr}-\frac{c}{k-1})\mutil_{\msr}+(w_\ell^T\mutil_{\ell}-(1-c))\mutil_{\ell}+\sqrt{\frac{k}{n}}(\sqrt{\frac{k}{n}}w_\ell^Ta-s)a+\sqrt{\frac{k}{n}}(\sqrt{\frac{k}{n}}w_\ell^Tb-t)b \right] + (\lambda + n\sigma^2) w_\ell = 0 \nonumber \\
& \frac{n}{k}\left[\sum_{\msr \ne \ell' }(w_{\ell'}^T\mutil_{\msr}-\frac{c}{k-1})\mutil_{\msr}+(w_{\ell'}^T\mutil_{\ell'}-(1-c))\mutil_{\ell'}+\sqrt{\frac{k}{n}}(\sqrt{\frac{k}{n}}w_{\ell'}^Ta-t)a+\sqrt{\frac{k}{n}}(\sqrt{\frac{k}{n}}w_{\ell'}^Tb-s)b \right] + (\lambda + n\sigma^2) w_{\ell'} = 0
\label{eq:kkt_l2}
\end{align}

Hence, $w_\ell$ belongs in the span of $\mutil_1, \dots, \mutil_k, a, b$. It is also easy to see from the symmetry in the data distribution that swapping $\mu_\ell \leftrightarrow \mu_{\tilde{\ell}}$ leads to getting the distributions of the solutions $w_\ell \leftrightarrow w_{\tilde{\ell}}$ swapped as well. This implies that $w_\ell^T \mutil_\ell = w_{\tilde{\ell}}^T \mutil_{\tilde{\ell}}$ and $w_{\ell'}^T \mutil_\ell = w_{\tilde{\ell'}}^T \mutil_{\tilde{\ell}}$ for all $1 \le \ell, \tilde{\ell}, \ell', \tilde{\ell'} \le k$ such that $\ell' \ne \ell, \tilde{\ell'} \ne \tilde{\ell}$. Therefore, $w_\ell$ and $w_{\ell'}$ can be written as 
\begin{align}
     & w_\ell = \gamma \sum_{\msr \ne \ell }\mutil_{\msr}+ \zeta \mutil_{\ell} + \alpha a + \beta b \nonumber \\
     & w_{\ell'}  = \gamma \sum_{\msr \ne \ell' }\mutil_{\msr}+ \zeta \mutil_{\ell'} + \beta a + \alpha b
    \label{eq:w_l}
\end{align} 

Denote $\lambda' = \lambda + n\sigma^2$. In terms of this notation, the first part of (\ref{eq:kkt_l2}) turns into:
\begin{align*} 
& \frac{n}{k}\left[\sum_{\msr \ne \ell }(w_\ell^T\mutil_{\msr}-\frac{c}{k-1})\mutil_{\msr}+(w_\ell^T\mutil_{\ell}-(1-c))\mutil_{\ell}+\sqrt{\frac{k}{n}}(\sqrt{\frac{k}{n}}w_\ell^Ta-s)a+\sqrt{\frac{k}{n}}(\sqrt{\frac{k}{n}}w_\ell^Tb-t)b \right] = - \lambda' (\gamma \sum_{\msr \ne \ell }\mutil_{\msr}+ \zeta \mutil_{\ell} + \alpha a + \beta b) 
\end{align*}

This leads us to
$$\begin{cases}
    & \frac{n}{k}(w_\ell^T\mutil_{\msr}-\frac{c}{k-1}) = -\lambda' \gamma \\
    & \frac{n}{k}(w_\ell^T\mutil_{\ell}-(1-c)) = -\lambda' \zeta \\
    & \sqrt{\frac{n}{k}}(\sqrt{\frac{k}{n}}w_\ell^Ta-s) = - \lambda' \alpha \\
    & \sqrt{\frac{n}{k}}(\sqrt{\frac{k}{n}}w_\ell^Tb-t) = - \lambda' \beta  \\
\end{cases}
$$

Plugging in (\ref{eq:w_l}) again we have: 
$$\begin{cases}
    & \frac{n}{k}\left[((k-2)r+1+\frac{k\sigma^2}{n})d\gamma+\zeta rd-\frac{c}{k-1}\right] + \lambda' \gamma = 0 \\
    & \frac{n}{k}\left[(k-1)rd\gamma+\zeta d(1+\frac{k\sigma^2}{n})-(1-c)\right] + \lambda' \zeta = 0 \\
    & \sqrt{\frac{n}{k}}(\sqrt{\frac{k}{n}}\alpha \sigma^2 d-s) + \lambda' \alpha = 0 \\
    & \sqrt{\frac{n}{k}}(\sqrt{\frac{k}{n}}\beta\sigma^2 d-t) + \lambda' \beta = 0
\end{cases}
$$

Denote $\tilde{\lambda} = \frac{\lambda k}{n} + \sigma^2k$. We obtain: 

$$\begin{cases}
    & \left[((k-2)r+1+\frac{k\sigma^2}{n})d + \tilde{\lambda} \right] \gamma + \zeta rd =\frac{c}{k-1} \\
    & (k-1)rd\gamma + (d(1+\frac{k\sigma^2}{n}) + \tilde{\lambda})\zeta  = 1-c   \\
    & \alpha = \frac{s\sqrt{nk}}{n\tilde{\lambda}+\sigma^2dk} \\
    & \beta = \frac{t\sqrt{nk}}{n\tilde{\lambda}+\sigma^2dk}
\end{cases}
$$

Defining $\Delta = \left[((k-2)r+1+\frac{k\sigma^2}{n})d + \tilde{\lambda}\right](d(1+\frac{k\sigma^2}{n}) + \tilde{\lambda}) - (k-1)(rd)^2$ and solving the equations above yields: 

$$\begin{cases}
    & \gamma =  \frac{1}{\Delta}\left[(d(1+\frac{k\sigma^2}{n}) + \tilde{\lambda})\frac{c}{k-1} -rd(1-c)  \right]\\
    &  \zeta = \frac{1}{\Delta}\left[-crd + (1-c)(((k-2)r+1+\frac{k\sigma^2}{n})d + \tilde{\lambda}) \right] \\
    & \alpha = \frac{s\sqrt{nk}}{n\tilde{\lambda}+\sigma^2dk} \\
    & \beta = \frac{t\sqrt{nk}}{n\tilde{\lambda}+\sigma^2dk}
\end{cases}
$$

Finally, one can use the equalities above along with (\ref{eq:w_l}) to conclude:

\begin{align*}
    & \mu_l^T(w_\ell - w_{\ell'}) = d(\zeta - \gamma)(1 - r) \\
    & \|w_\ell - w_{\ell'}\|_2^2 = 2d(\zeta - \gamma)^2(1 + \frac{\sigma^2k}{n} - r) + 2(\alpha - \beta) ^ 2\sigma ^ 2d
\end{align*}

\end{proof}

\subsection{Proof of Theorem \ref{thm: optimal_reg}}

\begin{theorem}
    Assume that either $\sigma^2 = o(n)$ or $c = 0$. Then taking $f(\cdot) = \|\cdot\|_2^2$ and $\lambda \to \infty$ yields performance that approaches the best possible among classification methods based on regularized regression with a convex regularizer $f(w)$ satisfying $f(w) \ge q(d)\|w\|_2$ for some $q: \mathbb{N} \to \mathbb{R}_+$ when $n \to \infty$. 
\end{theorem}

\begin{proof}
    Note that due to the symmetry w.r.t. to the permutation of $\mutil_1, \dots , \mutil_k$ in the objective we can write 
\begin{align*}
     & w_\ell = \gamma \sum_{\msr \ne \ell }\mutil_{\msr}+ \zeta \mutil_{\ell} + w_\ell^{\perp}\\
     & w_{\ell'}  = \gamma \sum_{\msr \ne \ell' }\mutil_{\msr}+ \zeta \mutil_{\ell'} + w_{\ell'}^{\perp}
\end{align*} 
Here $\gamma$ and $\alpha$ are some scalars depending on $\lambda$ and $w_\ell^{\perp}, w_{\ell'}^{\perp}$ are the projections of $w_\ell$ and  $w_{\ell'}$ respectively onto the subspace orthogonal to the span of $\mutil_1, \dots, \mutil_k$. Denote $w_-^{\perp} = w_{\ell}^{\perp} - w_{\ell'}^{\perp}$ and $w_- = w_\ell - w_{\ell'} = (\zeta - \gamma) (\mutil_{\ell} - \mutil_{\ell'}) + w_-^{\perp}$. If $c = 0$, then $w_-^{\perp} = 0$ and we are done. Otherwise we have that $\sigma^2 = o(n)$. Then to evaluate the classification error we need to evaluate the following quantity and the bigger this quantity the lower the error is: 
$$\frac{\mu_l^Tw_-}{\|w_-\|} = \frac{(\zeta - \gamma) \mu_l^T(\mutil_{\ell} - \mutil_{\ell'})}{\sqrt{ (\zeta - \gamma)^2 \|\mutil_{\ell} - \mutil_{\ell'}\|^2 + \|w_-^{\perp}\|^2}} = \frac{d(1-r)}{\sqrt{2d(1 + \frac{k \sigma^2}{n} - r) + \frac{\|w_-^{\perp}\|^2}{(\zeta - \gamma)^2}}} \le \frac{d(1-r)}{\sqrt{2d(1 + \frac{k \sigma^2}{n} - r)}}$$

Let us evaluate the performance for the solution for $f(\cdot) = \|\cdot\|_2^2$ now. Using the notation and the expressions from the previous section and taking $\lambda$ to go to $\infty$ we deduce that $$\lim_{\lambda \to \infty} \frac{\|w_-^{\perp}\|^2}{(\zeta - \gamma)^2} = \lim_{\lambda \to \infty} 2\frac{(\alpha - \beta) ^ 2\sigma ^ 2d}{(\zeta - \gamma)^2} = \lim_{\lambda \to \infty} 2k\frac{(s - t) ^ 2\sigma ^ 2d}{n\tilde{\lambda}(\zeta - \gamma)^2} = \lim_{\lambda \to \infty} 2k\frac{(s - t) ^ 2\sigma ^ 2d \tilde{\lambda}^2}{n\tilde{\lambda}^2(\frac{c}{k-1} - 1)^2} = O(\frac{d}{n}) = O(1)$$ Thus, since we are interested in studying the asymptotics of the error this term can be dropped from the denominator compared to $2d(1 + \frac{k \sigma^2}{n} - r)$ as $n \to \infty$ and we are done.
\end{proof}

\section{Proof of Theorem \ref{thm: l1}}
We begin this section with the following lemmas that will be of use:
\begin{lemma}
    Let $f: \bbR^d \rightarrow \bbR$ be a differentiable function, with $f(X) = 0$ on $\{X: a^T X = c\}$, then for an $X \sim \calN(0,I)$ we have:
    \begin{align*}
        \bbE X_1 f(X) \mathds{1}(a^TX\ge c) =  \bbE \mathds{1}(a^T X\ge c) \partial_{X_1} f(X) 
    \end{align*}
\end{lemma}
\begin{proof}
    We will consider the following three cases and apply integration by parts to each one of them:
    \begin{itemize}
        \item $a_1 > 0$
        \item $a_1 < 0$
        \item $a_1 = 0$
    \end{itemize}
    In the first case:
    \begin{align*}
        \bbE X_1 f(X) \mathds{1}(a^TX\ge c) = \frac{1}{\sqrt{(2\pi)^d}} \int_\bbR \exp(-\frac{x_2^2}{2}) dx_2 ... \int_\bbR \exp(-\frac{x_d^2}{2}) dx_d \int_{-\frac{1}{a_1}(c-\sum_{i=2}^d a_i x_i)} ^{\infty} x_1 f(x_1,...,x_d) \exp(-\frac{x_1^2}{2}) dx = \\
        \frac{1}{\sqrt{(2\pi)^d}} \int_\bbR \exp(-\frac{x_2^2}{2}) dx_2 ... \int_\bbR \exp(-\frac{x_d^2}{2}) dx_d (f(-\frac{1}{a_1}(c-\sum_{i=2}^d a_i x_i, x_2,..,x_d)) \exp(-\frac{(-\frac{1}{a_1}(c-\sum_{i=2}^d a_i x_i)^2}{2}) - 0) + \\ \frac{1}{\sqrt{(2\pi)^d}} \int_\bbR \exp(-\frac{x_2^2}{2}) dx_2 ... \int_\bbR \exp(-\frac{x_d^2}{2}) dx_d \int_{-\frac{1}{a_1}(c-\sum_{i=2}^d a_i x_i)} ^{\infty}  \partial_{x_1} f(x_1,...,x_d) \exp(-\frac{x_1^2}{2}) dx = \bbE \mathds{1}(a^T X\ge c) \partial_{X_1} f(X) 
    \end{align*}
     In the second case:
    \begin{align*}
        \bbE X_1 f(X) \mathds{1}(a^TX\ge c) = \frac{1}{\sqrt{(2\pi)^d}} \int_\bbR \exp(-\frac{x_2^2}{2}) dx_2 ... \int_\bbR \exp(-\frac{x_d^2}{2}) dx_d \int_{-\infty} ^{-\frac{1}{a_1}(c-\sum_{i=2}^d a_i x_i)} x_1 f(x_1,...,x_d) \exp(-\frac{x_1^2}{2}) dx = \\
        \frac{1}{\sqrt{(2\pi)^d}} \int_\bbR \exp(-\frac{x_2^2}{2}) dx_2 ... \int_\bbR \exp(-\frac{x_d^2}{2}) dx_d (f(-\frac{1}{a_1}(c-\sum_{i=2}^d a_i x_i, x_2,..,x_d)) \exp(-\frac{(-\frac{1}{a_1}(c-\sum_{i=2}^d a_i x_i)^2}{2}) - 0) + \\ \frac{1}{\sqrt{(2\pi)^d}} \int_\bbR \exp(-\frac{x_2^2}{2}) dx_2 ... \int_\bbR \exp(-\frac{x_d^2}{2}) dx_d \int_{-\infty} ^  {-\frac{1}{a_1}(c-\sum_{i=2}^d a_i x_i)}\partial_{x_1} f(x_1,...,x_d) \exp(-\frac{x_1^2}{2}) dx = \bbE \mathds{1}(a^T X\ge c) \partial_{X_1} f(X) 
    \end{align*}
In the third case we get the same as in the two previous ones if we assign $-\frac{1}{a_1}(c-\sum_{i=2}^d a_i x_i)^2 := \infty$
\end{proof}
\begin{lemma} \label{lem: stein's}
    Using the previous assumption on f, then for an invertible $\calR$, and $X \sim \calN(0, \calR)$
     \begin{align*}
        \bbE X_1 f(X) \mathds{1}(a^TX\ge c) =  \sum_{i=1}^d \bbE X_1X_i \bbE \mathds{1}(a^T X\ge c) \partial_{X_i} f(X) 
    \end{align*}
\end{lemma}
\begin{proof}
    Let $Y:= \calR^{-\frac{1}{2}} X$ be the whitened version of $X$, then
    \begin{align*}
        \bbE X_1 f(X) \mathds{1}(a^TX\ge c) = \bbE (\calR^{\frac{1}{2}})_1^T  Y f(\calR^{\frac{1}{2}}Y) \mathds{1}((\calR^{\frac{1}{2}}a)^TY\ge c) = (\calR^{\frac{1}{2}})_1 ^T \bbE Y f(\calR^{\frac{1}{2}}Y) \mathds{1}((\calR^{\frac{1}{2}}a)^TY\ge c)
    \end{align*}
    Now note that if $X$ is such that $a^T X = c$, then $(\calR^{\frac{1}{2}}a)^TY = c$, thus we can apply the previous lemma to the $Y \sim \calN(0,I)$, and have
    \begin{align*}
        &(\calR^{\frac{1}{2}})_1 ^T \bbE Y f(\calR^{\frac{1}{2}}Y) \mathds{1}((\calR^{\frac{1}{2}}a)^TY\ge c) = (\calR^{\frac{1}{2}})_1 ^T \bbE  \calR^{\frac{1}{2}} \nabla f(\calR^{\frac{1}{2}}Y) \mathds{1}((\calR^{\frac{1}{2}}a)^TY\ge c) = (\calR)_1 ^ T \bbE \nabla f(X) \mathds{1}(a^T X \ge c) = \\
        & \sum_{i=1}^d \bbE X_1X_i \bbE \mathds{1}(a^T X\ge c) \partial_{X_i} f(X) 
    \end{align*}
\end{proof}
Now we are ready to prove the main theorem of this section. For the sake of completeness, we present Theorem \ref{thm : solutions_match} here:
\begin{theorem}(LASSO Regression) \\
Under the assumptions (A1)-(A3), for any $\epsilon>0$, there exists $\lambda_0$ such that  for all $\lambda \ge \lambda_0$, the classification error corresponding to $f(\cdot) = \|\cdot\|_1$ can be described as follows as $n \to \infty$:
$$
|P_{e, \lambda}  - Q_k (\frac{n\sqrt{d}(1-r) (\gamma_2 - \gamma_1) R(\Xi) }{k|\Delta|(\Xi)}) | < \epsilon
$$
Where
\begin{align*}
 & \Xi^2:=\frac{n^2}{k^2} (k-1) \gamma_1^2(1+\frac{k\sigma^2}{n}+(k-2)r)  + \gamma_2^2 (1+\frac{k\sigma^2}{n})) + \frac{n^2}{k^2}(2 \gamma_1 \gamma_2 (k-1)r  +\frac{k}{n}\sigma^2 (\gamma_3^2+\gamma_4^2)) \\
  &R(\Xi) := 2 Q(\frac{\lambda}{\Xi}) \\
  &\Delta^2(\Xi) := \frac{\Xi^2}{\pi \sqrt{1-\Omega^2}}\int_{G > \frac{\lambda}{\Xi}} \int_{G' > \frac{\lambda}{\Xi}} (G- G')^2 \exp(-\frac{G^2 + G^{'2} - 2 \Omega G G'}{2(1-\Omega^2)}) dG dG'  \\
    & + \frac{2\Xi^2}{\pi \sqrt{1-\Omega^2}}\int_{G > \frac{\lambda}{\Xi}} \int_{|G'| < \frac{\lambda}{\Xi}} (G-\frac{\lambda}{\Xi})^2 \exp(-\frac{G^2 + G^{'2} - 2 \Omega G G'}{2(1-\Omega^2)}) dG dG' \\
    & + \frac{\Xi^2}{\pi \sqrt{1-\Omega^2}}\int_{G > \frac{\lambda}{\Xi}} \int_{G' < -\frac{\lambda}{\Xi}} (G-G'-2\frac{\lambda}{\Xi})^2 \exp(-\frac{G^2 + G^{'2} - 2 \Omega G G'}{2(1-\Omega^2)}) dG dG' \\
     &\Omega:= \frac{n^2}{k^2 \Xi^2 } \left[\gamma_1^2 (k-2)  (1+ \frac{k\sigma^2}{n} + (k-1) r) + (\gamma_1^2 + \gamma_2^2) r + 2\gamma_1 \gamma_2 ( 1+ \frac{k\sigma^2}{n} + (k-2)r ) + 2\frac{k}{n} \gamma_3 \gamma_4 \sigma^2\right]
\end{align*}
Where $\gamma_1, \gamma_2, \gamma_3, \gamma_4$ are the solutions to the following scalar optimization problem:
\begin{align*}
&\max_{\gamma_{1},\gamma_{2},\gamma_{3},\gamma_{4}} -\frac{d (\Xi^2 + \lambda^2) R(\Xi)}{4\rho} + \frac{ d \Xi \lambda }{2\rho \sqrt{2\pi}} \exp(-\frac{\lambda^2}{2\Xi^2})+ \frac{n}{k} \left[-c\gamma_1-(k-1)\frac{\gamma_1^2}{4}-(1-c)\gamma_2-\frac{\gamma_2^2}{4}-s\gamma_3-\frac{\gamma_3^2}{4}-t\gamma_4-\frac{\gamma_4^2}{4}\right]
\end{align*}
Moreover, only  $\left \lceil d R(\Xi) \right \rceil = \left \lceil  2d Q(\frac{\lambda}{\Xi})\right \rceil$ entries of each classifier are non-zero.
\label{thm: l_1}
\end{theorem}
\begin{proof}
We break the proof into several steps, described in the following sections.
\subsection{Scalarization of $\phiu$}
First, consider the following optimization for an arbitrary $\rho = n\sigma^2$, to obtain the final result, we plug-in $\rho = n \sigma^2$.
\begin{align*}
& \min_{w_\ell, w_{\ell'}} \frac{n}{k}\left[\sum_{\msr \ne \ell }(w_\ell^T\mutil_{\msr}-\frac{c}{k-1})^2+(w_\ell^T\mutil_{\ell}-(1-c))^2+(\sqrt{\frac{k}{n}}w_\ell^Ta-s)^2+(\sqrt{\frac{k}{n}}w_\ell^Tb-t)^2\right] + \rho \|w_\ell\|^2 + \lambda f(w_\ell) + \nonumber \\
&  + \frac{n}{k}\left[ \sum_{\msr \ne \ell' }(w_{\ell'}^T\mutil_{\msr}-\frac{c}{k-1})^2+(w_{\ell'}^T\mutil_{\ell'}-(1-c))^2+(\sqrt{\frac{k}{n}}w_{\ell'}^Ta-t)^2+(\sqrt{\frac{k}{n}}w_{\ell'}^Tb-s)^2 \right]+ \rho \|w_{\ell'}\|^2 + \lambda f(w_{\ell'})
\end{align*}
Taking the Fenchel dual of each term and making use of the symmetry of the model, we arrive at the following:
\begin{align*}
& \min_{w_\ell, w_{\ell'}} \max_{\gamma}\frac{n}{k}\left[\sum_{\msr \ne \ell }[\gamma_{\ell,1}(w_\ell^T\mutil_{\msr}-\frac{c}{k-1})-\frac{\gamma_{\ell,1}^2}{4}]+\gamma_{\ell,2}(w_\ell^T\mutil_{\ell}-(1-c))-\frac{\gamma_{\ell,2}^2}{4}+\gamma_{\ell,3}(\sqrt{\frac{k}{n}}w_\ell^Ta-s)-\frac{\gamma_{\ell,3}^2}{4} \right] + \\ &+ \frac{n}{k}\left[\gamma_{\ell,4}(\sqrt{\frac{k}{n}}w_\ell^Tb-t)-\frac{\gamma_{\ell,4}^2}{4}\right]+ \frac{n}{k}\left[ \sum_{\msr \ne \ell' }[\gamma_{\ell',1}(w_{\ell'}^T\mutil_{\msr}-\frac{c}{k-1})-\frac{\gamma_{\ell',1}^2}{4}]+\gamma_{\ell',2}(w_{\ell'}^T\mutil_{\ell'}-(1-c))-\frac{\gamma_{\ell',2}^2}{4}\right] +  \nonumber \\
& +\frac{n}{k}\left[\gamma_{\ell',3}(\sqrt{\frac{k}{n}}w_{\ell'}^Ta-t)-\frac{\gamma_{\ell',3}^2}{4}+\gamma_{\ell',4}(\sqrt{\frac{k}{n}}w_{\ell'}^Tb-s) - \frac{\gamma_{\ell',4}^2}{4} \right] + \rho \|w_\ell\|^2 + \lambda f(w_\ell) +\rho \|w_{\ell'}\|^2+ \lambda f(w_{\ell'}) 
\end{align*}
Due to symmetry, we only consider the optimization over $w_{\ell}$:
\begin{align*}
   \min_{w_{\ell}} \frac{n}{k} \left[\gamma_{\ell,1}\sum_{r \neq \ell} \mutil_r +\gamma_{\ell,2} \mutil_{\ell} + \gamma_{\ell,3} \sqrt{\frac{k}{n}} a +  \gamma_{\ell,4} \sqrt{\frac{k}{n}} b\right]^Tw_{\ell} + \rho \|w_\ell\|^2 + \lambda \|w_{\ell}\|_1
\end{align*}
We first tackle an equivalent scalar optimization:
\begin{align*}
    \min_x ax + \lambda |x| + \rho x^2
\end{align*}
Taking derivative and solving for $x$ yields
\begin{align*}
     a + \lambda  \partial |x| + 2\rho x = 0
\end{align*}
\begin{align*}
    x^* = \frac{1}{2\rho}ST(-a;\lambda):= \begin{cases}
        \frac{1}{2\rho}(-a-\lambda) & -a-\lambda > 0 \\
        \frac{1}{2\rho}(-a+\lambda) & -a+\lambda < 0 \\
        0 & o.w
    \end{cases} 
\end{align*}
Where $ST(x;c)$ is defined as follows:
\begin{align*}
    ST(x;c) := \begin{cases}
        x - c & x > c\\
        x + c & x < -c \\
        0 & o.w
    \end{cases}
\end{align*}
The objective value would be
\begin{align*} \min_x ax + \lambda |x| + \rho x^2 =
    \begin{cases}
        \frac{1}{2\rho} ((-a^2 - a \lambda) + (- \lambda a - \lambda^2 ) + \frac{1}{2} (a+\lambda)^2) &  -a-\lambda > 0 \\ 
        \frac{1}{2\rho} ((-a^2 + a \lambda) + ( \lambda a - \lambda^2 ) + \frac{1}{2} (a-\lambda)^2) & -a+\lambda < 0 \\
        0 & o.w 
    \end{cases}  \\ = \begin{cases}
        -\frac{1}{4\rho}(a+\lambda)^2 &  -a-\lambda > 0 \\
        -\frac{1}{4\rho}(a-\lambda)^2 &  -a+\lambda < 0 \\
        0 & o.w 
    \end{cases} = -\frac{1}{4\rho} ST^2(-a;\lambda)
\end{align*}

Now we apply this result to our problem. First let us define the following vectors for both of the classes $\ell, \ell'$.
\begin{align*}
    \Psi_\ell := \frac{n}{k} \left[\gamma_{\ell,1}\sum_{r \neq \ell} \mutil_r +\gamma_{\ell,2} \mutil_{\ell} + \gamma_{\ell,3} \sqrt{\frac{k}{n}} a +  \gamma_{\ell,4} \sqrt{\frac{k}{n}} b\right] \\
     \Psi_{\ell'} := \frac{n}{k}\left[ \gamma_{\ell',1} \sum_{\msr \ne \ell' } \mutil_{\msr}+\gamma_{\ell',2}\mutil_{\ell'}+\gamma_{\ell',3}\sqrt{\frac{k}{n}}a+\gamma_{\ell',4}\sqrt{\frac{k}{n}}b\right]
\end{align*}
Leveraging the assumptions on $\mu$'s, we observe that $Var(\mutil_{i,r}) = 1+\frac{k \sigma^2}{n}$ and $\bbE \mutil_{i,\msr} \mutil_{i,\msr'} = r $. Therefore, taking $G\sim \calN(0,I)$, we have
\begin{align*}
   \Psi_{\ell} \stackrel{d}{=} \frac{n}{k} \sqrt{(k-1) \gamma_{\ell,1}^2(1+\frac{k\sigma^2}{n}+(k-2)r)  + \gamma_{\ell,2}^2 (1+\frac{k\sigma^2}{n}) + 2 \gamma_{\ell,1} \gamma_{\ell,2} (k-1)r  +\frac{k}{n}\sigma^2 (\gamma_{\ell,3}^2+\gamma_{\ell,4}^2)} G := \Xi G
\end{align*}
Then plugging in
\begin{align*}
    w^*_{\ell} = \frac{1}{2\rho} ST(-\frac{n}{k} (\gamma_{\ell,1}\sum_{r \neq \ell} \mutil_r +\gamma_{\ell,2} \mutil_{\ell} + \gamma_{\ell,3} \sqrt{\frac{k}{n}} a +  \gamma_{\ell,4} \sqrt{\frac{k}{n}} b); \lambda) =  \frac{1}{2\rho} ST(-\Psi_\ell; \lambda)
\end{align*}

Using this result, the optimization over w would achieve the following value:
\begin{align*}
    \min_{w_{\ell}} \frac{n}{k} \left[\gamma_{\ell,1}\sum_{r \neq \ell} \mutil_r +\gamma_{\ell,2} \mutil_{\ell} + \gamma_{\ell,3} \sqrt{\frac{k}{n}} a +  \gamma_{\ell,4} \sqrt{\frac{k}{n}} b\right]^Tw_{\ell} + \rho \|w_\ell\|^2 + \lambda \|w_{\ell}\|_1 = \frac{-1}{4\rho} \| ST(-\Psi_\ell; \lambda)\|^2
\end{align*}
Now using the concentration properties we calculate the expectation:
\begin{align*}
   &\bbE  \| ST(-\Psi_\ell; \lambda)\|^2 = d \bbE  ST^2(\Xi G; \lambda) = d \bbE (\Xi G - \lambda)^2 \mathds{1}(\Xi G > \lambda ) + d \bbE (\Xi G + \lambda)^2 \mathds{1}(\Xi G < -\lambda ) =\\
   &2 \Xi^2 d \bbE G^2  \mathds{1}(\Xi G > \lambda ) + 2 d \lambda^2 \bbE \mathds{1}(\Xi G > \lambda ) - 2  d\Xi \lambda \bbE G \mathds{1}(\Xi G > \lambda ) + 2 d \Xi \lambda \bbE G \mathds{1}(\Xi G < -\lambda ) = \\
   &2 \Xi^2 d \bbE G^2  \mathds{1}(G > \frac{\lambda}{\Xi} ) + 2d \lambda^2 Q(\frac{\lambda}{\Xi}) - 4 d \Xi \lambda \bbE G \mathds{1}(G > \frac{\lambda}{\Xi}  ) = 2 \Xi^2 d \bbE G^2  \mathds{1}(G > \frac{\lambda}{\Xi} ) + 2d \lambda^2 Q(\frac{\lambda}{\Xi}) -\frac{ 4 d \Xi \lambda }{\sqrt{2\pi}} \exp(-\frac{\lambda^2}{2\Xi^2}) = \\ &2 d \lambda^2 Q(\frac{\lambda}{\Xi}) -\frac{ 4 d \Xi \lambda }{\sqrt{2\pi}} \exp(-\frac{\lambda^2}{2\Xi^2}) + 2\Xi^2 d (Q(\frac{\lambda}{\Xi}) + \frac{\lambda}{\Xi \sqrt{2\pi}} \exp(-\frac{\lambda^2}{2\Xi^2})) = \\ &2d Q(\frac{\lambda}{\Xi})  (\Xi^2 + \lambda^2) -\frac{ 2 d \Xi \lambda }{\sqrt{2\pi}} \exp(-\frac{\lambda^2}{2\Xi^2})
\end{align*}

Plugging back in, we get the following optimization
\begin{align*}
&\max_{\gamma_{\ell,1},\gamma_{\ell,2},\gamma_{\ell,3},\gamma_{\ell,4}} \frac{n}{k} \left[-c\gamma_{\ell,1}-(k-1)\frac{\gamma_{\ell,1}^2}{4}-(1-c)\gamma_{\ell,2}-\frac{\gamma_{\ell,2}^2}{4}-s\gamma_{\ell,3}-\frac{\gamma_{\ell,3}^2}{4}-t\gamma_{\ell,4}-\frac{\gamma_{\ell,4}^2}{4}\right] - \\ &- \frac{d (\Xi^2 + \lambda^2) R(\Xi)}{4\rho} + \frac{ d \Xi \lambda }{2\rho \sqrt{2\pi}} \exp(-\frac{\lambda^2}{2\Xi^2})
\end{align*}
Where the quantity
\begin{align*}
    R(\Xi) = \bbP(|\Psi_\ell| = \Xi G \ge \lambda) = 2 Q(\frac{\lambda}{\Xi})
\end{align*}
is the the fraction of nonzero entries of the optimal $w$.

\subsection{Derivation of the classification error}
First note that by the derivation from previous section, $\|w_{\ell}-w_{\ell'}\|^2 = \frac{1}{4\rho^2}\|ST(-\Psi_\ell; \lambda) - ST(-\Psi_{\ell'}; \lambda)\|^2$ and $\mu_\ell^T (w_\ell - w_{\ell'}) = \frac{1}{2\rho} \mu_\ell^T (ST(-\Psi_\ell; \lambda) - ST(-\Psi_{\ell'}; \lambda))$. Furthermore, It can be seen that $\gamma_{\ell,1}=\gamma_{\ell',1} = \gamma_1$ and $\gamma_{\ell,2}=\gamma_{\ell',2} = \gamma_2$ and $\gamma_{\ell,3}=\gamma_{\ell',4}= \gamma_3$ and $\gamma_{\ell,4}=\gamma_{\ell',3}= \gamma_4$ due to symmetry. In fact, one has
\begin{align*}
    (\Psi_{\ell,i}, \Psi_{\ell',i}) \sim \calN(0, \begin{pmatrix}
        \sigma^2_\ell & r_\ell\\
        r_\ell & \sigma^2_{\ell'}
    \end{pmatrix} )
\end{align*}
Where by symmetry, $\sigma^2_\ell = \sigma^2_{\ell'} = \Xi^2$, and for the correlation,
\begin{align*}
    &\Omega := \frac{1}{\Xi^2 d}\bbE \Psi_\ell^T \Psi_{\ell'} = \frac{n^2}{k^2 \Xi^2} \left[\gamma_1^2 (\sum_{r \neq \ell, \ell'} \mutil_r) ^2 + (\gamma_1^2 + \gamma_2^2) \mutil_\ell \mutil_{\ell'}\right] +  \frac{n}{k \Xi^2} \gamma_3 \gamma_4 \left[a^2 + b ^2\right] + \\
    & + \frac{n^2}{k^2 \Xi^2} \left[\gamma_1 \gamma_2 (\mutil_\ell^2 + \mutil_{\ell'}^2) + \gamma_1^2 ( \sum_{r \neq \ell, \ell'} \mutil_\ell \mutil_r +  \sum_{r \neq \ell, \ell'} \mutil_{\ell'} \mutil_r ) + \gamma_1 \gamma_2 (\sum_{r \neq \ell, \ell'} \mutil_\ell \mutil_r +  \sum_{r \neq \ell, \ell'} \mutil_{\ell'} \mutil_r ) \right]  = \\
     &= \frac{n^2}{k^2 \Xi^2} \left[\gamma_1^2 (k-2)  (1+ \frac{k\sigma^2}{n} + (k-3) r) + (\gamma_1^2 + \gamma_2^2) r + 2\gamma_1 \gamma_2 ( 1+ \frac{k\sigma^2}{n}) + 2(k-2)r \gamma_1^2 + 2 \gamma_1 \gamma_2 (k-2) r + 2\frac{k}{n} \gamma_3 \gamma_4 \sigma^2\right] = \\
     &\frac{n^2}{k^2 \Xi^2 } \left[\gamma_1^2 (k-2)  (1+ \frac{k\sigma^2}{n} + (k-3) r) + (\gamma_1^2 + \gamma_2^2) r + 2\gamma_1 \gamma_2 ( 1+ \frac{k\sigma^2}{n} + (k-2)r ) + 2(k-2)r \gamma_1^2 + 2\frac{k}{n} \gamma_3 \gamma_4 \sigma^2\right] =\\
     &\frac{n^2}{k^2 \Xi^2 } \left[\gamma_1^2 (k-2)  (1+ \frac{k\sigma^2}{n} + (k-1) r) + (\gamma_1^2 + \gamma_2^2) r + 2\gamma_1 \gamma_2 ( 1+ \frac{k\sigma^2}{n} + (k-2)r ) + 2\frac{k}{n} \gamma_3 \gamma_4 \sigma^2\right]
\end{align*}
Thus 
\begin{align*}
    (\Psi_{\ell,i}, \Psi_{\ell',i} ) \sim \calN(0, \begin{pmatrix}
        \Xi^2 & \Omega \Xi^2 \\
        \Omega \Xi^2 & \Xi^2
    \end{pmatrix} )
\end{align*}
And we can take $ (\Psi_{\ell,i}, \Psi_{\ell',i} ) \stackrel{d}{=} (G,G' ) \sim \calN(0, \begin{pmatrix}
        1 & \Omega \\
        \Omega  & 1
    \end{pmatrix} )$. Now we focus on the calculation of the classification error which requires the characterization of $\frac{\mu_\ell^T (w^*_{\ell'}-w^*_{\ell})}{\|w^*_{\ell'}-w^*_{\ell}\|}$. Beginning with $\|w^*_{\ell'}-w^*_{\ell}\|$, we have from concentration.
\begin{align*}
    &\frac{1}{d}\|ST(-\Psi_\ell; \lambda) - ST(-\Psi_{\ell'}; \lambda)\|^2 =  \bbE \left[(\Psi_{\ell,i} -\Psi_{\ell',i})^2 \mathds{1}(-\Psi_{\ell,i} > \lambda, -\Psi_{\ell',i} > \lambda)\right] + \\& +  \bbE \left[(\Psi_{\ell,i} -\Psi_{\ell',i})^2 \mathds{1}(-\Psi_{\ell,i} <- \lambda, -\Psi_{\ell',i} <- \lambda)\right] + 
    \bbE \left[(\Psi_{\ell,i} + \lambda)^2 \mathds{1}(-\Psi_{\ell,i} > \lambda, |\Psi_{\ell',i}| < \lambda)\right] + \\ & + \bbE \left[(\Psi_{\ell',i} + \lambda)^2 \mathds{1}(|\Psi_{\ell,i}| < \lambda, -\Psi_{\ell',i} > \lambda)\right] +  
    \bbE \left[(\Psi_{\ell,i} -\lambda)^2 \mathds{1}(-\Psi_{\ell,i} < -\lambda, |\Psi_{\ell',i}| < \lambda)\right] + \\& + \bbE \left[(\Psi_{\ell',i} -\lambda)^2 \mathds{1}(|\Psi_{\ell,i}| < \lambda, -\Psi_{\ell',i} < -\lambda)\right] + 
     \bbE \left[(\Psi_{\ell',i} -\Psi_{\ell,i} - 2 \lambda)^2 \mathds{1}(-\Psi_{\ell,i} > \lambda, -\Psi_{\ell',i} < -\lambda)\right] + \\& + \bbE \left[(\Psi_{\ell,i} -\Psi_{\ell',i} + 2 \lambda)^2 \mathds{1}(-\Psi_{\ell,i} < -\lambda, -\Psi_{\ell',i} > \lambda)\right] = \\ 
     &2 \bbE \left[(\Psi_{\ell,i} -\Psi_{\ell',i})^2 \mathds{1}(-\Psi_{\ell,i} > \lambda, -\Psi_{\ell',i} > \lambda)\right] + 4 \bbE \left[(\Psi_{\ell,i} -\lambda)^2 \mathds{1}(-\Psi_{\ell,i} < -\lambda, |\Psi_{\ell',i}| < \lambda)\right] + \\ & 2 \bbE \left[(\Psi_{\ell',i} -\Psi_{\ell,i} - 2 \lambda)^2 \mathds{1}(-\Psi_{\ell,i} > \lambda, -\Psi_{\ell',i} < -\lambda)\right] = 2 \Xi^2 \bbE \left[(G-G')^2 \mathds{1} (G < -\frac{\lambda}{\Xi}, G' < -\frac{\lambda}{\Xi}) \right] + \\ & + 4 \Xi^2 \bbE\left[(G-\frac{\lambda}{\Xi})^2 \mathds{1}(G>\frac{\lambda}{\Xi}, |G'| < \frac{\lambda}{\Xi})\right]  + 2 \Xi^2 \bbE \left[(G-G'+2\frac{\lambda}{\Xi})^2\mathds{1}(G<-\frac{\lambda}{\Xi}, G'> \frac{\lambda}{\Xi})\right]
\end{align*}

We calculate each term separately. Let $(G,G') \sim \calN(0, \calR) $ with $\calR = \begin{pmatrix}
    1 & \Omega \\
    \Omega & 1
\end{pmatrix}$
and $\calR^{-1}= \frac{1}{1-\Omega^2}\begin{pmatrix}
    1 & -\Omega \\
    -\Omega & 1
\end{pmatrix}$. Using this, one can write

\begin{align*}
    2 \Xi^2 \bbE \left[(G-G')^2 \mathds{1} (G < -\frac{\lambda}{\Xi}, G' < -\frac{\lambda}{\Xi}) \right] = \frac{\Xi^2}{\pi \sqrt{1-\Omega^2}}\int_{G > \frac{\lambda}{\Xi}} \int_{G' > \frac{\lambda}{\Xi}} (G- G')^2 \exp(-\frac{G^2 + G^{'2} - 2 \Omega G G'}{2(1-\Omega^2)})
\end{align*}
\begin{align*}
    4 \Xi^2 \bbE\left[(G-\frac{\lambda}{\Xi})^2 \mathds{1}(G>\frac{\lambda}{\Xi}, |G'| < \frac{\lambda}{\Xi})\right] = \frac{2\Xi^2}{\pi \sqrt{1-\Omega^2}}\int_{G > \frac{\lambda}{\Xi}} \int_{|G'| < \frac{\lambda}{\Xi}} (G-\frac{\lambda}{\Xi})^2 \exp(-\frac{G^2 + G^{'2} - 2 \Omega G G'}{2(1-\Omega^2)})
\end{align*}

\begin{align*}
    2 \Xi^2 \bbE \left[(G-G'+2\frac{\lambda}{\Xi})^2\mathds{1}(G<-\frac{\lambda}{\Xi}, G'> \frac{\lambda}{\Xi})\right] = \frac{\Xi^2}{\pi \sqrt{1-\Omega^2}}\int_{G > \frac{\lambda}{\Xi}} \int_{G' < -\frac{\lambda}{\Xi}} (G-G'-2\frac{\lambda}{\Xi})^2 \exp(-\frac{G^2 + G^{'2} - 2 \Omega G G'}{2(1-\Omega^2)})
\end{align*}

Combining results yields
\begin{align*}
    &\frac{1}{d}\|ST(-\Psi_\ell; \lambda) - ST(-\Psi_{\ell'}; \lambda)\|^2 = \frac{\Xi^2}{\pi \sqrt{1-\Omega^2}}\int_{G > \frac{\lambda}{\Xi}} \int_{G' > \frac{\lambda}{\Xi}} (G- G')^2 \exp(-\frac{G^2 + G^{'2} - 2 \Omega G G'}{2(1-\Omega^2)}) dG dG'  \\
    & + \frac{2\Xi^2}{\pi \sqrt{1-\Omega^2}}\int_{G > \frac{\lambda}{\Xi}} \int_{|G'| < \frac{\lambda}{\Xi}} (G-\frac{\lambda}{\Xi})^2 \exp(-\frac{G^2 + G^{'2} - 2 \Omega G G'}{2(1-\Omega^2)}) dG dG' \\
    & + \frac{\Xi^2}{\pi \sqrt{1-\Omega^2}}\int_{G > \frac{\lambda}{\Xi}} \int_{G' < -\frac{\lambda}{\Xi}} (G-G'-2\frac{\lambda}{\Xi})^2 \exp(-\frac{G^2 + G^{'2} - 2 \Omega G G'}{2(1-\Omega^2)}) dG dG'
\end{align*}

For the inner product, we utilize Lemma \ref{lem: stein's} and calculate the each term in $\mu_\ell^T (w^*_{\ell'}-w^*_{\ell})$ separately. First, we have
\begin{align*}
    &\mu_\ell^T w^*_\ell = \frac{d}{2\rho} \bbE \mu_{\ell,i} ST(-\Psi_{\ell,i};\lambda) = \frac{d}{2\rho} \sum_{\msr = 1}^k \bbE  \mu_{\ell,i} \mu_{\msr,i} \left[\bbE \mathds{1}(-\Psi_{\ell,i}>\lambda) \partial_{\mu_{\msr,i}}  (-\Psi_{\ell,i}-\lambda) + \bbE \mathds{1}(-\Psi_{\ell,i}<-\lambda)\partial_{\mu_{\msr,i}}  (-\Psi_{\ell,i}+\lambda) \right] = \\ & -\frac{dn}{2\rho k} (k-1) r \gamma_1 \bbE \mathds{1}(|\Psi_{\ell,i}| > \lambda)  - \frac{dn}{2\rho k} \gamma_2 \bbE \mathds{1}(|\Psi_{\ell,i}| > \lambda) = - \frac{dn}{2\rho k} R ((k-1)r \gamma_1 + \gamma_2)
\end{align*}
Also similarly,
\begin{align*}
     &\mu_\ell^T w^*_{\ell'} = \frac{d}{2\rho} \bbE \mu_{\ell,i} ST(-\Psi_{\ell',i};\lambda) = \frac{d}{2\rho} \sum_{\msr = 1}^k \bbE  \mu_{\ell,i} \mutil_{\msr,i} \left[\bbE \mathds{1}(-\Psi_{\ell',i}>\lambda) \partial_{\mu_{\msr,i}}  (-\Psi_{\ell',i}-\lambda) + \bbE \mathds{1}(-\Psi_{\ell',i}<-\lambda)\partial_{\mu_{\msr,i}}  (-\Psi_{\ell,i}+\lambda) \right] = \\
     & - r\frac{nd}{2\rho k}  \gamma_2 \bbE \mathds{1}(|\Psi_{\ell',i}| > \lambda) -  \frac{nd}{2\rho k}  (k-2) r \gamma_1 \bbE \mathds{1}(|\Psi_{\ell',i}| > \lambda) - \frac{nd}{2\rho k}  \gamma_1 \bbE \mathds{1}(|\Psi_{\ell',i}| > \lambda) = -\frac{nd}{2\rho k} R ( \gamma_1 (1 + (k-2) r ) + \gamma_2 r )
\end{align*}
Thus this yields the following
\begin{align*}
    \mu_\ell^T (w^*_{\ell'}-w^*_{\ell}) = -\frac{nd}{2\rho k} R (\gamma_1 (1 + (k-2) r ) + \gamma_2 r - (k-1)r \gamma_1 - \gamma_2) = -\frac{nd}{2\rho k}R (1-r) (\gamma_1 - \gamma_2)
\end{align*}
Therefore, summarizing,
\begin{align*}
    & \frac{\mu_\ell^T (w^*_{\ell'}-w^*_{\ell})}{\|w^*_{\ell'}-w^*_{\ell}\|} =  \frac{\frac{n\sqrt{d}}{k}R (1-r) (\gamma_2 - \gamma_1)}{|\Delta|(\Xi)} \\
    & \Delta^2(\Xi) := \frac{\Xi^2}{\pi \sqrt{1-\Omega^2}}\int_{G > \frac{\lambda}{\Xi}} \int_{G' > \frac{\lambda}{\Xi}} (G- G')^2 \exp(-\frac{G^2 + G^{'2} - 2 \Omega G G'}{2(1-\Omega^2)}) dG dG'  \\
    & + \frac{2\Xi^2}{\pi \sqrt{1-\Omega^2}}\int_{G > \frac{\lambda}{\Xi}} \int_{|G'| < \frac{\lambda}{\Xi}} (G-\frac{\lambda}{\Xi})^2 \exp(-\frac{G^2 + G^{'2} - 2 \Omega G G'}{2(1-\Omega^2)}) dG dG' \\
    & + \frac{\Xi^2}{\pi \sqrt{1-\Omega^2}}\int_{G > \frac{\lambda}{\Xi}} \int_{G' < -\frac{\lambda}{\Xi}} (G-G'-2\frac{\lambda}{\Xi})^2 \exp(-\frac{G^2 + G^{'2} - 2 \Omega G G'}{2(1-\Omega^2)}) dG dG'
\end{align*}
Now we apply Theorem \ref{thm : solutions_match} and observe that with high probability, for a large enough $\lambda$,
$$
|P_{e, \lambda}  - Q_k (\frac{n\sqrt{d}(1-r) (\gamma_2 - \gamma_1) R(\Xi) }{k|\Delta|(\Xi)}) | < \epsilon
$$
\end{proof}
\section{Proof of Corollary \ref{cor: l1}}
By the result from Theorem $\ref{cor: dist_match}$, we observe that the sparsity rate predicted for $\phiu$ and $\tilde{\Phi}$ can be made arbitrarily close as $\lambda \rightarrow \infty$. 

Now to prove the Corollary \ref{cor: l1}, observe that we have from the previous theorem:
    \begin{align*}
        dR = \frac{2\rho k \mu_\ell^T (w^*_{\ell'}-w^*_{\ell})}{n(1-r)(\gamma_2 - \gamma_1)}
    \end{align*}
    First note that by the Fenchel dual argument, 
    \begin{align*}
        \gamma_1 = 2(\mutil_\ell^Tw^*_{\ell'}-\frac{c}{k-1}), \quad \gamma_2 = 2(\mutil_\ell^Tw^*_{\ell}-(1-c))
    \end{align*}
    Now, by taking $\lambda$ large enough, one can have $|\mutil_\ell^Tw^*_{\ell'}| < \epsilon$ and $|\mutil_\ell^Tw^*_{\ell}| < \epsilon$. This is possible since both terms are present in the optimization and increasing $\lambda$ would make $w^*_{\ell} = 0$ a feasible solution. This implies $\gamma_2 - \gamma_1 \xrightarrow{\lambda \rightarrow \infty} \frac{ck}{k-1}-1$. Furthermore, the condition $\sigma \ll \sqrt{n}$ ensures that $\mutil_\ell^Tw^*_{\ell'}$ and $\mu_\ell^Tw^*_{\ell'}$ stay close. Thus $\mu_\ell^T (w^*_{\ell'}-w^*_{\ell})\xrightarrow{\lambda\rightarrow\infty}0$. Therefore 
    \begin{align*}
        dR = \frac{2\rho k \mu_\ell^T (w^*_{\ell'}-w^*_{\ell})}{n(1-r)(\gamma_2 - \gamma_1)} \xrightarrow{\lambda \rightarrow \infty} 0
    \end{align*}

\section{Proof of Theorem \ref{thm: linf}}
\begin{theorem}
    ($\|\cdot\|_\infty$ - Regularization)
    Under the assumptions (A1)-(A3), the classification error corresponding to $f(\cdot) = \|\cdot\|_\infty$ can be described as follows in the large $\lambda$ regime with probability approaching $1$ as $n \to \infty$:
    $$
    \lim_{\lambda \to \infty} P_{e, \lambda} =  \lim_{\lambda \to \infty} Q_k ( \frac{n\sqrt{d} (1-r) (\gamma_2 - \gamma_1) (1-R(\Xi)) }{2k\rho|\Delta|(\Xi)})
    $$
    Where
    \begin{align*}
         &\Xi^2:=\frac{n^2}{k^2} (k-1) \gamma_1^2(1+\frac{k\sigma^2}{n}+(k-2)r)  + \gamma_2^2 (1+\frac{k\sigma^2}{n})) + \frac{n^2}{k^2}(2 \gamma_1 \gamma_2 (k-1)r  +\frac{k}{n}\sigma^2 (\gamma_3^2+\gamma_4^2)) \\
        & R(\Xi) := 2 Q(\frac{2\rho \delta}{\Xi \lambda}) \\
        & \Delta(\Xi)^2 :=  \frac{\Xi^2 }{8 \rho^2 \pi \sqrt{1-\Omega^2}} \int_{|G|<\frac{2\rho \delta}{\Xi \lambda}} \int_{|G'|<\frac{2\rho \delta}{\Xi \lambda}} (G- G')^2 \exp(-\frac{G^2 + G^{'2} - 2 \Omega G G'}{2(1-\Omega^2)}) dG dG' + \\
    &\frac{4 \delta^2}{\lambda^2 \pi  \sqrt{1-\Omega^2}} \int_{G<-\frac{2\rho \delta}{\Xi \lambda}} \int_{G'>\frac{2\rho \delta}{\Xi \lambda}} \exp(-\frac{G^2 + G^{'2} - 2 \Omega G G'}{2(1-\Omega^2)}) dG dG'+ \\
    &\frac{ \Xi^2}{4 \rho^2 \pi \sqrt{1-\Omega^2}} \int_{|G|<\frac{2\rho \delta}{\Xi \lambda}} \int_{G'<-\frac{2\rho \delta}{\Xi \lambda}} (G+ \frac{2\rho \delta}{\lambda \Xi})^2 \exp(-\frac{G^2 + G^{'2} - 2 \Omega G G'}{2(1-\Omega^2)}) dG dG' + \\
    &\frac{\Xi^2}{4 \rho^2 \pi \sqrt{1-\Omega^2}} \int_{|G|<\frac{2\rho \delta}{\Xi \lambda}} \int_{G'>\frac{2\rho \delta}{\Xi \lambda}} (G- \frac{2\rho \delta}{\lambda \Xi})^2 \exp(-\frac{G^2 + G^{'2} - 2 \Omega G G'}{2(1-\Omega^2)}) dG dG' \\
    & \Omega:= \frac{n^2}{k^2 \Xi^2 } \left[\gamma_1^2 (k-2)  (1+ \frac{k\sigma^2}{n} + (k-1) r) + (\gamma_1^2 + \gamma_2^2) r + 2\gamma_1 \gamma_2 ( 1+ \frac{k\sigma^2}{n} + (k-2)r ) + 2\frac{k}{n} \gamma_3 \gamma_4 \sigma^2\right]
    \end{align*}
In which $\gamma_1, \gamma_2, \gamma_3, \gamma_4$ are the solutions to the following scalar optimization problem:
\begin{align*}
        & \min_{\delta \ge 0}  \max_{\gamma_1,\gamma_2,\gamma_3,\gamma_4} \delta +\frac{d\delta^2 \rho R(\Xi)}{\lambda^2} - (1-R) \frac{d \Xi^2}{4\rho} - \frac{d\delta\Xi}{\lambda\sqrt{2\pi}}\exp(-\frac{2\rho^2 \delta^2}{\Xi^2 \lambda^2}) \\
        & + \frac{n}{k} \left[-c\gamma_1-(k-1)\frac{\gamma_1^2}{4}-(1-c)\gamma_2-\frac{\gamma_2^2}{4}-s\gamma_3-\frac{\gamma_3^2}{4}-t\gamma_4-\frac{\gamma_4^2}{4}\right]
\end{align*} 
With $\rho:=n\sigma^2$. Moreover, $\zeta$ of the coordinates of the weights are equal to $\frac{\delta}{\lambda}$ and $\zeta$ of the weights are equal to $-\frac{\delta}{\lambda}$ and the rest of them take values between $-\frac{\delta}{\lambda}$ and $\frac{\delta}{\lambda}$ with $\zeta = \left \lceil \frac{R(\Xi)}{2} \right \rceil = \left \lceil Q(\frac{2\delta\rho}{\lambda \Xi}) \right \rceil$
\label{thm: linf_complete}
\end{theorem}
\begin{proof}
Similar to the $\|\cdot\|_1$ case, we proceed by dividing the proof into several steps.
\subsection{Scalarization of $\phiu$}
Introducing $\delta$ for $\lambda \|w_{\ell}\|_{\infty}$, we have for arbitrary $\rho > 0$
\begin{align*}
\min_{\lambda \|w_{\ell}\|_{\infty}\le \delta} \max_{\gamma}\frac{n}{k}\left[\sum_{\msr \ne \ell }[\gamma_{\ell,1}(w_\ell^T\mutil_{\msr}-\frac{c}{k-1})-\frac{\gamma_{\ell,1}^2}{4}]+\gamma_{\ell,2}(w_\ell^T\mutil_{\ell}-(1-c))-\frac{\gamma_{\ell,2}^2}{4} \right] + \\
 + \frac{n}{k} \left [\gamma_{\ell,3}(\sqrt{\frac{k}{n}}w_\ell^Ta-s) - \frac{\gamma_{\ell,3}^2}{4}+\gamma_{\ell,4}(\sqrt{\frac{k}{n}}w_\ell^Tb-t)-\frac{\gamma_{\ell,4}^2}{4}\right] + \rho\|w_\ell\|^2 +  \delta
\end{align*}
The optimization over $w_\ell$ is
\begin{align*}
   \min_{\|w_{\ell}\|_{\infty} \le \frac{\delta}{\lambda}} \frac{n}{k} \left[\gamma_{\ell,1}\sum_{r \neq \ell} \mutil_r +\gamma_{\ell,2} \mutil_{\ell} + \gamma_{\ell,3} \sqrt{\frac{k}{n}} a +  \gamma_{\ell,4} \sqrt{\frac{k}{n}} b\right]^Tw_{\ell}  + \rho \|w_\ell\|^2 = \min_{\|w_{\ell}\|_{\infty} \le \frac{\delta}{\lambda}} \Psi_\ell ^ T w_{\ell}  + \rho \|w_\ell\|^2
\end{align*}
For the counterpart scalar optimization $\min_{|x|<b} ax + \rho x^2$ we have:
\begin{align*}
    x^* &= \begin{cases}
        b & \frac{-a}{2\rho} > b \\
        \frac{-a}{2\rho} & -b\le \frac{-a}{2\rho}\le b \\
        -b & \frac{-a}{2\rho} < -b
    \end{cases}\\
    \min_{|x|<b} ax + \rho x^2 &=  \begin{cases}
        ab + \rho b^2 & \frac{-a}{2\rho} > b \\
        \frac{-a^2}{4\rho} & -b\le \frac{-a}{2\rho}\le b \\
        -ab + \rho b^2 & \frac{-a}{2\rho} < -b
    \end{cases}
\end{align*}
Applying to our problem, then for the solution we have:
\begin{align*}
    w^*_{\ell,i} = \begin{cases}
        \frac{\delta}{\lambda} & \frac{-\Psi_{\ell,i}}{2\rho} >  \frac{\delta}{\lambda} \\
        \frac{-\Psi_{\ell,i}}{2\rho} & -\frac{\delta}{\lambda} \le \frac{-\Psi_{\ell,i}}{2\rho} \le  \frac{\delta}{\lambda} \\
        -\frac{\delta}{\lambda} & \frac{-\Psi_{\ell,i}}{2\rho} <  -\frac{\delta}{\lambda}
    \end{cases} = \frac{-\Psi_{\ell,i}}{2\rho} - (\frac{-\Psi_{\ell,i}}{2\rho} - \frac{\delta}{\lambda}) \mathds{1}(\frac{-\Psi_{\ell,i}}{2\rho} >  \frac{\delta}{\lambda}) - (\frac{-\Psi_{\ell,i}}{2\rho} + \frac{\delta}{\lambda}) \mathds{1}(\frac{-\Psi_{\ell,i}}{2\rho} <  -\frac{\delta}{\lambda}) 
\end{align*}
Let 
\begin{align*}
    R = \bbP(|\Psi_\ell| = \Xi G \ge \frac{2\rho \delta}{\lambda}) = 2 Q(\frac{2\rho \delta}{\Xi \lambda})
\end{align*}
All in all
\begin{align*}
    &\min_{\|w_{\ell}\|_{\infty} \le \frac{\delta}{\lambda}} \Psi_\ell ^ T w_{\ell}  + \rho \|w_\ell\|^2 = d \frac{\delta}{\lambda} \bbE \left[ - \Xi G + \rho \frac{\delta}{\lambda}  \right]\mathds{1}(\frac{\Xi G}{2 \rho} > \frac{\delta}{\lambda}) + d \frac{\delta}{\lambda} \bbE \left[ \Xi G + \rho \frac{\delta}{\lambda}  \right]\mathds{1}(\frac{\Xi G}{2 \rho} <- \frac{\delta}{\lambda}) \\
    &- \frac{d}{4\rho} \bbE \Xi^2 G^2  \mathds{1}(-\frac{\delta}{\lambda}\le \frac{\Xi G}{2 \rho} \le \frac{\delta}{\lambda}) = \\
    &\frac{Rd\delta^2 \rho}{\lambda^2} + \frac{d\delta \Xi}{\lambda} \bbE ( -G \mathds{1}(G > \frac{2 \rho\delta}{\Xi \lambda}) + G \mathds{1}(G < -\frac{2 \rho\delta}{\Xi \lambda})) - \frac{d \Xi^2}{4\rho} \bbE G^2 \mathds{1}( -\frac{2 \rho\delta}{\Xi \lambda} \le G \le  \frac{2 \rho\delta}{\Xi \lambda})) = \\
    &\frac{Rd\delta^2 \rho}{\lambda^2} - \frac{d\delta\Xi}{\lambda}\bbE |G| \mathds{1}(|G| > \frac{2 \rho\delta}{\Xi \lambda}) - \frac{d \Xi^2}{4\rho} \bbE G^2 \mathds{1}( -\frac{2 \rho\delta}{\Xi \lambda} \le G \le  \frac{2 \rho\delta}{\Xi \lambda}) = \\
    &\frac{Rd\delta^2 \rho}{\lambda^2} - \frac{2d\delta\Xi}{\lambda\sqrt{2\pi}}\exp(-\frac{2\rho^2 \delta^2}{\Xi^2 \lambda^2}) -  \frac{d\Xi^2}{4 \rho} \left[1-2Q(\frac{2 \rho\delta}{\Xi \lambda}) - \sqrt{\frac{2}{\pi}} \frac{2 \rho\delta}{\Xi \lambda} \exp(-\frac{2\rho^2 \delta^2}{\Xi^2 \lambda^2})\right] = \\
    &\frac{Rd\delta^2 \rho}{\lambda^2} - \frac{d \Xi^2}{4\rho} - \frac{d\delta\Xi}{\lambda\sqrt{2\pi}}\exp(-\frac{2\rho^2 \delta^2}{\Xi^2 \lambda^2}) + \frac{Rd\Xi^2}{4\rho}
\end{align*}
Where we had by definition $\Xi =  \frac{n}{k} \sqrt{(k-1) \gamma_{\ell,1}^2(1+\frac{k\sigma^2}{n}+(k-2)r)  + \gamma_{\ell,2}^2 (1+\frac{k\sigma^2}{n}) + 2 \gamma_{\ell,1} \gamma_{\ell,2} (k-1)r  +\frac{k}{n}\sigma^2 (\gamma_{\ell,3}^2+\gamma_{\ell,4}^2)} $
Summarizing, the objective is:
\begin{align*}
    \min_{\delta \ge 0}  \max_{\gamma_{\ell,1},\gamma_{\ell,2},\gamma_{\ell,3},\gamma_{\ell,4}} \frac{n}{k} \left[-c\gamma_{\ell,1}-(k-1)\frac{\gamma_{\ell,1}^2}{4}-(1-c)\gamma_{\ell,2}-\frac{\gamma_{\ell,2}^2}{4}-s\gamma_{\ell,3}-\frac{\gamma_{\ell,3}^2}{4}-t\gamma_{\ell,4}-\frac{\gamma_{\ell,4}^2}{4}\right] + \delta \\ +\frac{Rd\delta^2 \rho}{\lambda^2} - (1-R) \frac{d \Xi^2}{4\rho} - \frac{d\delta\Xi}{\lambda\sqrt{2\pi}}\exp(-\frac{2\rho^2 \delta^2}{\Xi^2 \lambda^2})
\end{align*}
\subsection{Derivation of the classification error}
For the classification error similar to the previous case, we tackle the problem by determining $\| w^*_{\ell} -  w^*_{\ell'} \|$ and $\mu_\ell^T (w^*_{\ell'}-w^*_{\ell})$. To start off, we have for $\| w^*_{\ell} -  w^*_{\ell'} \|$:
\begin{align*}
    &\| w^*_{\ell} -  w^*_{\ell'} \|^2 = d \bbE ( w^*_{\ell,i} -  w^*_{\ell',i} )^2 = d \bbE(\frac{\Psi_{\ell,i}}{2\rho} - \frac{\Psi_{\ell',i}}{2\rho})^2 \mathds{1}(|\frac{\Psi_{\ell,i}}{2\rho}| < \frac{\delta}{\lambda},|\frac{\Psi_{\ell',i}}{2\rho}| < \frac{\delta}{\lambda}) +\\ &+ 2d \bbE (\frac{\delta}{\lambda} + \frac{\delta}{\lambda})^2 \mathds{1}(-\frac{\Psi_{\ell,i}}{2\rho} > \frac{\delta}{\lambda}, -\frac{\Psi_{\ell',i}}{2\rho} <- \frac{\delta}{\lambda}) + 2d \bbE (\frac{\Psi_{\ell,i}}{2\rho} + \frac{\delta}{\lambda})^2 \mathds{1}(|\frac{\Psi_{\ell,i}}{2\rho}| < \frac{\delta}{\lambda},-\frac{\Psi_{\ell',i}}{2\rho} > \frac{\delta}{\lambda}) + \\
    & + 2d \bbE (\frac{\Psi_{\ell,i}}{2\rho} - \frac{\delta}{\lambda})^2 \mathds{1}(|\frac{\Psi_{\ell,i}}{2\rho}| < \frac{\delta}{\lambda},-\frac{\Psi_{\ell',i}}{2\rho} < -\frac{\delta}{\lambda}) = \\
    &d \frac{\Xi^2}{4\rho^2} \bbE (G - G')^2 \mathds{1}(|G| < \frac{2\rho \delta}{\Xi\lambda},|G'| < \frac{2\rho \delta}{\Xi\lambda}) + 8d \frac{\delta^2}{\lambda^2} \bbE \mathds{1}(-G > \frac{2\rho \delta}{\Xi\lambda},-G' < -\frac{2\rho \delta}{\Xi\lambda}) + \\
    &+ 2d \frac{\Xi^2}{4\rho^2} \bbE (G +  \frac{2\rho \delta}{\Xi\lambda})^2 \mathds{1}(|G| < \frac{2\rho \delta}{\Xi\lambda},-G' > \frac{2\rho \delta}{\Xi\lambda})+ 2d \frac{\Xi^2}{4\rho^2} \bbE (G -  \frac{2\rho \delta}{\Xi\lambda})^2 \mathds{1}(|G| < \frac{2\rho \delta}{\Xi\lambda},-G' < -\frac{2\rho \delta}{\Xi\lambda})
\end{align*}
Expanding each term
\begin{align*}
    & d \frac{\Xi^2}{4\rho^2} \bbE (G - G')^2 \mathds{1}(|G| < \frac{2\rho \delta}{\Xi\lambda},|G'| < \frac{2\rho \delta}{\Xi\lambda}) = \frac{d \Xi^2 }{8 \rho^2 \pi \sqrt{1-\Omega^2}} \int_{|G|<\frac{2\rho \delta}{\Xi \lambda}} \int_{|G'|<\frac{2\rho \delta}{\Xi \lambda}} (G- G')^2 \exp(-\frac{G^2 + G^{'2} - 2 \Omega G G'}{2(1-\Omega^2)}) dG dG' \\
    & 8d \frac{\delta^2}{\lambda^2} \bbE \mathds{1}(-G > \frac{2\rho \delta}{\Xi\lambda},-G' < -\frac{2\rho \delta}{\Xi\lambda}) = \frac{4 d \delta^2}{\lambda^2 \pi  \sqrt{1-\Omega^2}} \int_{G<-\frac{2\rho \delta}{\Xi \lambda}} \int_{G'>\frac{2\rho \delta}{\Xi \lambda}} \exp(-\frac{G^2 + G^{'2} - 2 \Omega G G'}{2(1-\Omega^2)}) dG dG'\\
    \end{align*}
\begin{align*}
     2d \frac{\Xi^2}{4\rho^2} \bbE (G +  \frac{2\rho \delta}{\Xi\lambda})^2 \mathds{1}(|G| < \frac{2\rho \delta}{\Xi\lambda},-G' & > \frac{2\rho \delta}{\Xi\lambda}) =  \\ &\frac{d \Xi^2}{4 \rho^2 \pi \sqrt{1-\Omega^2}} \int_{|G|<\frac{2\rho \delta}{\Xi \lambda}} \int_{G'<-\frac{2\rho \delta}{\Xi \lambda}} (G+ \frac{2\rho \delta}{\lambda \Xi})^2 \exp(-\frac{G^2 + G^{'2} - 2 \Omega G G'}{2(1-\Omega^2)}) dG dG' \\
      2d \frac{\Xi^2}{4\rho^2} \bbE (G -  \frac{2\rho \delta}{\Xi\lambda})^2 \mathds{1}(|G| < \frac{2\rho \delta}{\Xi\lambda},-G' &< -\frac{2\rho \delta}{\Xi\lambda})  =  \\ & \frac{d \Xi^2}{4 \rho^2 \pi \sqrt{1-\Omega^2}} \int_{|G|<\frac{2\rho \delta}{\Xi \lambda}} \int_{G'>\frac{2\rho \delta}{\Xi \lambda}} (G- \frac{2\rho \delta}{\lambda \Xi})^2 \exp(-\frac{G^2 + G^{'2} - 2 \Omega G G'}{2(1-\Omega^2)}) dG dG' 
\end{align*}
Combining all of the terms implies
\begin{align*}
    &\| w^*_{\ell} -  w^*_{\ell'} \|^2 = \frac{d \Xi^2 }{8 \rho^2 \pi \sqrt{1-\Omega^2}} \int_{|G|<\frac{2\rho \delta}{\Xi \lambda}} \int_{|G'|<\frac{2\rho \delta}{\Xi \lambda}} (G- G')^2 \exp(-\frac{G^2 + G^{'2} - 2 \Omega G G'}{2(1-\Omega^2)}) dG dG' + \\
    &\frac{4 d \delta^2}{\lambda^2 \pi  \sqrt{1-\Omega^2}} \int_{G<-\frac{2\rho \delta}{\Xi \lambda}} \int_{G'>\frac{2\rho \delta}{\Xi \lambda}} \exp(-\frac{G^2 + G^{'2} - 2 \Omega G G'}{2(1-\Omega^2)}) dG dG'+ \\
    &\frac{d \Xi^2}{4 \rho^2 \pi \sqrt{1-\Omega^2}} \int_{|G|<\frac{2\rho \delta}{\Xi \lambda}} \int_{G'<-\frac{2\rho \delta}{\Xi \lambda}} (G+ \frac{2\rho \delta}{\lambda \Xi})^2 \exp(-\frac{G^2 + G^{'2} - 2 \Omega G G'}{2(1-\Omega^2)}) dG dG' + \\
    &\frac{d \Xi^2}{4 \rho^2 \pi \sqrt{1-\Omega^2}} \int_{|G|<\frac{2\rho \delta}{\Xi \lambda}} \int_{G'>\frac{2\rho \delta}{\Xi \lambda}} (G- \frac{2\rho \delta}{\lambda \Xi})^2 \exp(-\frac{G^2 + G^{'2} - 2 \Omega G G'}{2(1-\Omega^2)}) dG dG'  
\end{align*}
For the inner product, we use the Lemma \ref{lem: stein's}:
\begin{align*}
    &\mu_\ell^T w^*_\ell = d \bbE \mu_{\ell,i} w^*_{\ell,i}) = d \sum_{\msr = 1}^k \bbE  \mu_{\ell,i} \mu_{\msr,i} \bbE \left[ \partial_{\mu_{\msr,i}} \frac{-\Psi_{\ell,i}}{2\rho} - 
    \mathds{1}(\frac{-\Psi_{\ell,i}}{2\rho} >  \frac{\delta}{\lambda})\partial_{\mu_{\msr,i}}(\frac{-\Psi_{\ell,i}}{2\rho} - \frac{\delta}{\lambda})  - \mathds{1}(\frac{-\Psi_{\ell,i}}{2\rho} <  -\frac{\delta}{\lambda}) \partial_{\mu_{\msr,i}}(\frac{-\Psi_{\ell,i}}{2\rho} + \frac{\delta}{\lambda})  \right] = \\
    & d (k-1) r \frac{-n \gamma_1 }{2k\rho } \bbE \left[ 1 - \mathds{1}(\frac{-\Psi_{\ell,i}}{2\rho} >  \frac{\delta}{\lambda}) -  \mathds{1}(\frac{-\Psi_{\ell,i}}{2\rho} <  -\frac{\delta}{\lambda})\right] + d \frac{-n\gamma_2}{2k\rho} \bbE \left[ 1 - \mathds{1}(\frac{-\Psi_{\ell,i}}{2\rho} >  \frac{\delta}{\lambda}) -  \mathds{1}(\frac{-\Psi_{\ell,i}}{2\rho} <  -\frac{\delta}{\lambda})\right] =  \\
    &-\frac{nd(1-R)}{2k\rho}((k-1)r \gamma_1 + \gamma_2)
\end{align*}
Moreover,
\begin{align*}
    \mu_\ell^T w^*_{\ell'} = d & \bbE \mu_{\ell,i} w^*_{\ell',i} = \\&=d \sum_{\msr = 1}^k \bbE  \mu_{\ell,i} \mu_{\msr,i} \bbE \left[ \partial_{\mu_{\msr,i}} \frac{-\Psi_{\ell',i}}{2\rho} - 
    \mathds{1}(\frac{-\Psi_{\ell',i}}{2\rho} >  \frac{\delta}{\lambda})\partial_{\mu_{\msr,i}}(\frac{-\Psi_{\ell',i}}{2\rho} - \frac{\delta}{\lambda})  - \mathds{1}(\frac{-\Psi_{\ell',i}}{2\rho} <  -\frac{\delta}{\lambda}) \partial_{\mu_{\msr,i}}(\frac{-\Psi_{\ell',i}}{2\rho} + \frac{\delta}{\lambda})  \right] = 
\end{align*}
\begin{align*}
    & d (k-2) r \frac{-n\gamma_1}{2k\rho} \bbE \left[ 1 - \mathds{1}(\frac{-\Psi_{\ell',i}}{2\rho} >  \frac{\delta}{\lambda}) -  \mathds{1}(\frac{-\Psi_{\ell',i}}{2\rho} <  -\frac{\delta}{\lambda})\right] + d  \frac{-n\gamma_1}{2k\rho} \bbE \left[ 1 - \mathds{1}(\frac{-\Psi_{\ell',i}}{2\rho} >  \frac{\delta}{\lambda}) -  \mathds{1}(\frac{-\Psi_{\ell',i}}{2\rho} <  -\frac{\delta}{\lambda})\right] +  \\ & + d \frac{-nr\gamma_2}{2k\rho} \bbE \left[ 1 - \mathds{1}(\frac{-\Psi_{\ell',i}}{2\rho} >  \frac{\delta}{\lambda}) -  \mathds{1}(\frac{-\Psi_{\ell',i}}{2\rho} <  -\frac{\delta}{\lambda})\right] = 
    -\frac{dn(1-R)}{2k\rho}(\gamma_1(1+(k-2)r)  + r\gamma_2)
\end{align*}
All in all
\begin{align*}
    \mu_\ell^T (w^*_{\ell'}-w^*_{\ell}) = -dn\frac{1-R}{2k\rho} (\gamma_1 (1 + (k-2) r ) + \gamma_2 r - (k-1)r \gamma_1 - \gamma_2) = -dn\frac{1-R}{2k\rho} (1-r)  (\gamma_1 - \gamma_2)
\end{align*}
Summarizing the results yields
\begin{align*}
    & \frac{\mu_\ell^T (w^*_{\ell'}-w^*_{\ell})}{\|w^*_{\ell'}-w^*_{\ell})\|} =  \frac{n\sqrt{d} (1-r) (\gamma_2 - \gamma_1) (1-R(\Xi)) }{2k\rho\Delta(\Xi)} \\
    & \Delta(\Xi)^2 :=  \frac{\Xi^2 }{8 \rho^2 \pi \sqrt{1-\Omega^2}} \int_{|G|<\frac{2\rho \delta}{\Xi \lambda}} \int_{|G'|<\frac{2\rho \delta}{\Xi \lambda}} (G- G')^2 \exp(-\frac{G^2 + G^{'2} - 2 \Omega G G'}{2(1-\Omega^2)}) dG dG' + \\
    &\frac{4 \delta^2}{\lambda^2 \pi  \sqrt{1-\Omega^2}} \int_{G<-\frac{2\rho \delta}{\Xi \lambda}} \int_{G'>\frac{2\rho \delta}{\Xi \lambda}} \exp(-\frac{G^2 + G^{'2} - 2 \Omega G G'}{2(1-\Omega^2)}) dG dG'+ \\
    &\frac{ \Xi^2}{4 \rho^2 \pi \sqrt{1-\Omega^2}} \int_{|G|<\frac{2\rho \delta}{\Xi \lambda}} \int_{G'<-\frac{2\rho \delta}{\Xi \lambda}} (G+ \frac{2\rho \delta}{\lambda \Xi})^2 \exp(-\frac{G^2 + G^{'2} - 2 \Omega G G'}{2(1-\Omega^2)}) dG dG' + \\
    &\frac{\Xi^2}{4 \rho^2 \pi \sqrt{1-\Omega^2}} \int_{|G|<\frac{2\rho \delta}{\Xi \lambda}} \int_{G'>\frac{2\rho \delta}{\Xi \lambda}} (G- \frac{2\rho \delta}{\lambda \Xi})^2 \exp(-\frac{G^2 + G^{'2} - 2 \Omega G G'}{2(1-\Omega^2)}) dG dG'  
\end{align*}
\end{proof}
\section{Proof of Corollary \ref{cor: linf}}
Similar to the previous case by the result from Theorem $\ref{cor: dist_match}$, we observe that the $\zeta$ predicted from $\phiu$ matches $\Phi$.

Moreover, from the previous theorem, we have
    \begin{align*}
       \frac{d (1-2 \zeta)}{2} =  \frac{k\rho \mu_\ell^T (w^*_{\ell'}-w^*_{\ell})}{ n (1-r) (\gamma_2 - \gamma_1)}  
    \end{align*}
    Similarly to the case $f(\cdot) = \|\cdot\|_1$, we have $\gamma_2 - \gamma_1 \xrightarrow{\lambda \rightarrow \infty} \frac{ck}{k-1}-1$ and $\mu_\ell^T (w^*_{\ell'}-w^*_{\ell})\xrightarrow{\lambda\rightarrow\infty}0$ which implies $\lim_{\lambda \to \infty} \zeta = \frac{1}{2}$.

\section{Additional Experiments}

Fig \ref{fig:l2_1200_600}, \ref{fig:l1_1200_600}, \ref{fig:linf_1200_600} showcase more plots for datasets generated from GMMs. They are produced in the same way, follow the same notation as the plots from the main body of the paper and turn out to be very similar to those qualitatively but we include them here for providing further evidence towards our observations. 

Moreover, in addition to synthetic data we experimented with MNIST as well. We took the number of classes $k = 5$ as before and then picked $100$ first points corresponding to each digit from $0$ to $4$ from the MNIST train set, so that the total number of data we train on is $n = 500$. After that, we varied $\lambda$ and trained classifiers for the constructed dataset by solving (\ref{eq:est}) using CVXPY for $f(\cdot) = \|\cdot\|_2^2$ and $f(\cdot) = \|\cdot\|_\infty$ and evaluated their performances by evaluating the test errors on the part of the MNIST test set corresponding to the selected digits. Same as for the GMMs, we observed that the solutions for $f(\cdot) = \|\cdot\|_\infty$ can be compressed to one bit and for large enough values of $\lambda$ and their performances are only marginally lower than the performances of the solutions corresponding to $f(\cdot) = \|\cdot\|_2^2$. We noticed that the picture does not depend much on $c$, so we provide the results we obtained for $c = 0.4$ here. The reader can find the results of this experiment in Fig. \ref{fig: MNIST}. The plot on the left hand side depicts the test errors of the true and compressed solutions for $f(\cdot) = \|\cdot\|_\infty$ and the plot on the right hand side depicts the test errors corresponding to $f(\cdot) = \|\cdot\|^2_2$).  

\label{sec: exp}

\begin{figure}%
    \centering
    \subfloat{\includegraphics[width = 3.25 in]{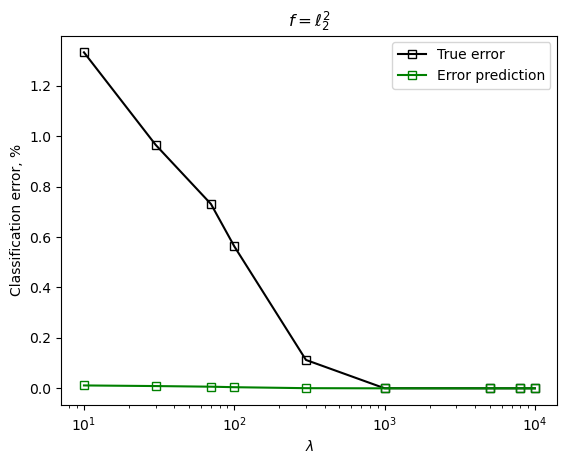}}%
    \subfloat{\includegraphics[width = 3.25 in]{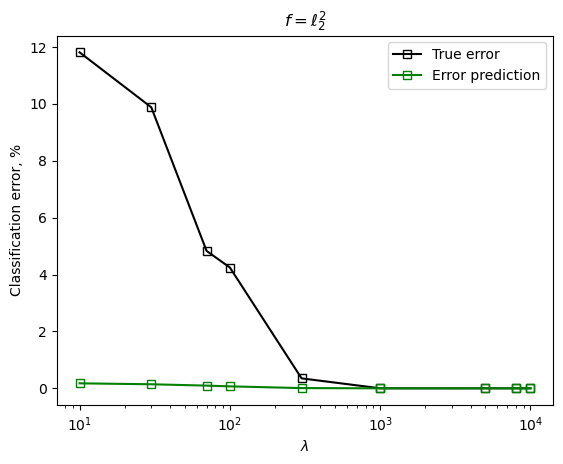}}%
    \caption{We took $d = 1200$, $n = 300$, $k = 5$, $r = 0.7$, $c =0.2$, $\sigma = 1$ for the first plot and  $d = 600$, $n = 300$, $k = 5$, $r = 0.7$, $c =0.2$, $\sigma = 1$ for the second. In both cases, the prediction underestimates the true error for the smaller values of $\lambda$ but matches it for the greater values of $\lambda$, as expected. } %
    \label{fig:l2_1200_600}
\end{figure}
\begin{figure}[htp]
    \centering
    \subfloat{\includegraphics[width = 3.25 in]{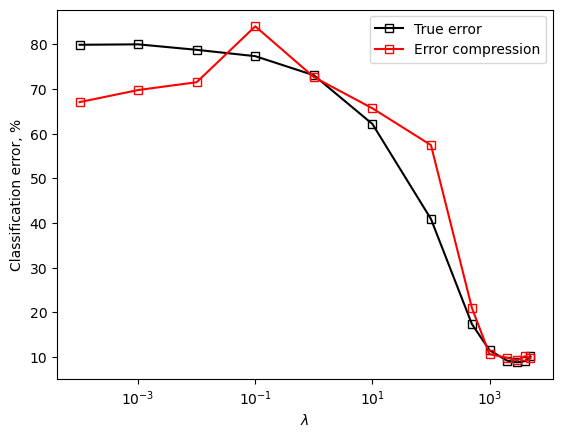}}
    \subfloat{\includegraphics[width = 3.25 in]{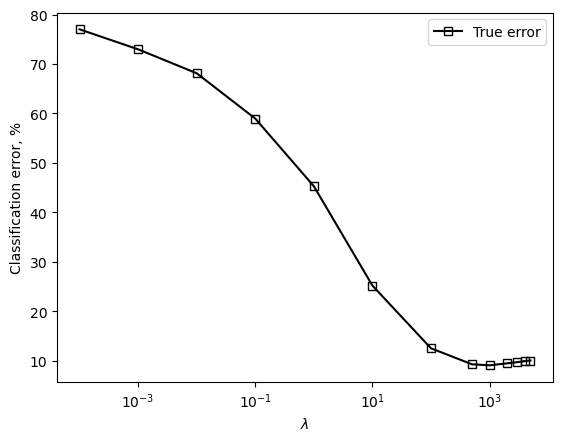}}
    \caption{MNIST dataset. Optimal errors for both classifiers are approximately equal to $0.09$. }
    \label{fig: MNIST}
\end{figure}

\begin{figure}[htp]
    \centering
    \subfloat{\includegraphics[width = 3.25 in]{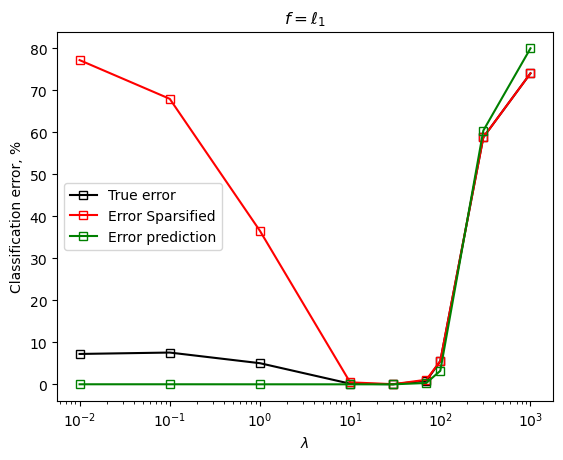}}
    \subfloat{\includegraphics[width = 3.25 in]{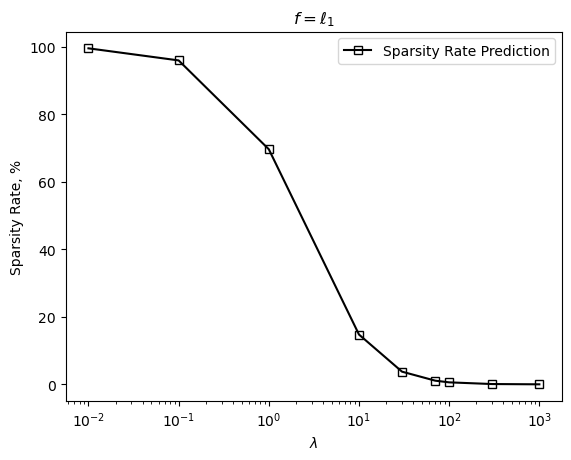}}
    
    \subfloat{\includegraphics[width = 3.25 in]{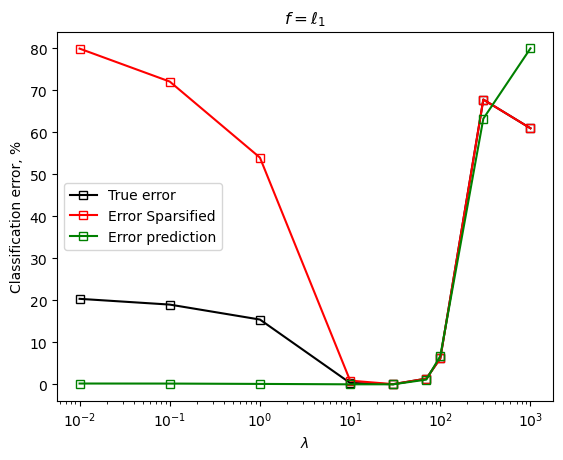}}
    \subfloat{\includegraphics[width = 3.25 in]{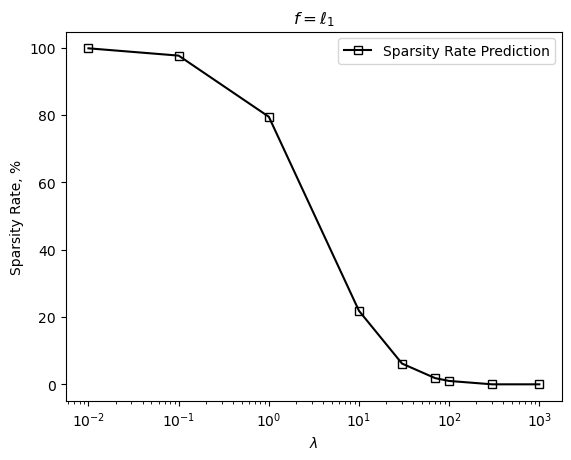}}
    \caption{We took $d = 1200$, $n = 300$, $k = 5$, $r = 0.7$, $c =0.2$ and $\sigma = 1$ for the plots on the top and $d = 600$, $n = 300$, $k = 5$, $r = 0.7$, $c =0.2$ for the plots on the bottom. They illustrate that for these parameters it is possible to sparsify the weights by 33X and 16X respectively, while keeping the classification error very low. }
    \label{fig:l1_1200_600}
\end{figure}

\begin{figure}[htp]
    \centering
    \subfloat{\includegraphics[width = 3.25 in]{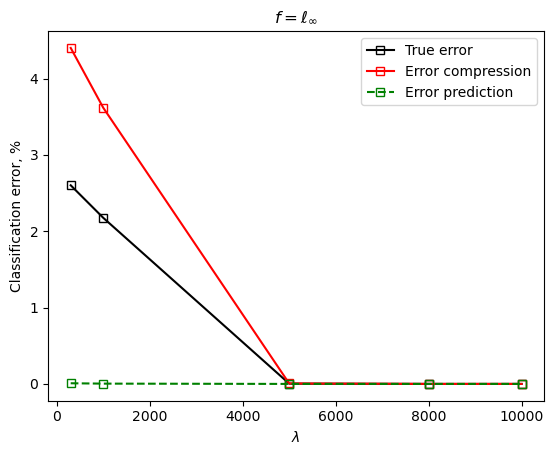}}
    \subfloat{\includegraphics[width = 3.25 in]{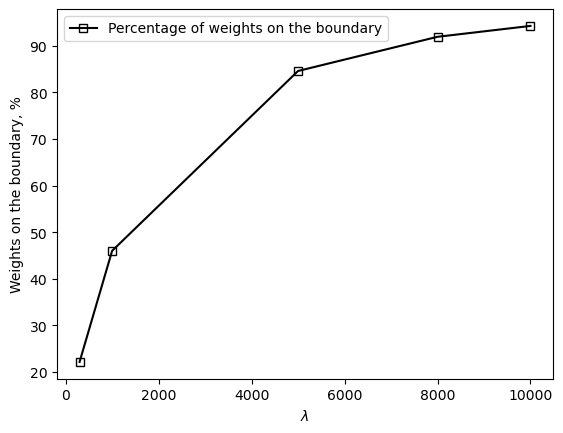}}
    
    \subfloat{\includegraphics[width = 3.25 in]{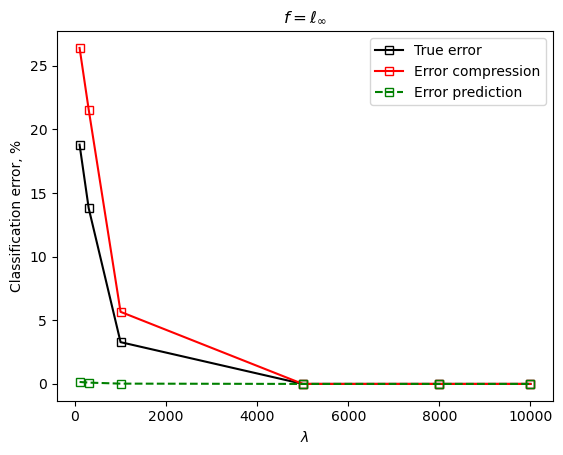}}
    \subfloat{\includegraphics[width = 3.25 in]{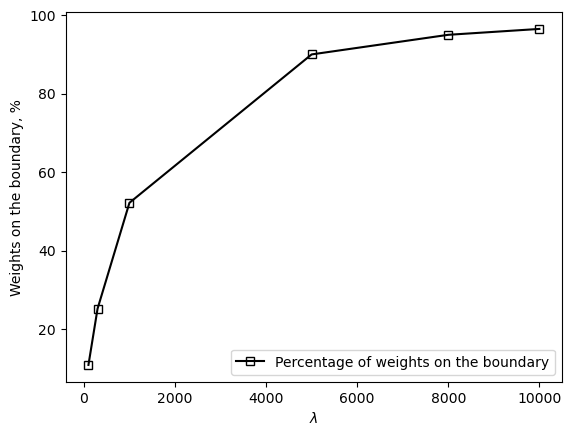}}
    \caption{We took $d = 1200$, $n = 300$, $k = 5$, $r = 0.7$, $c =0$, $\sigma = 1$ for the plots on the top and $d = 600$, $n = 300$, $k = 5$, $r = 0.7$, $c =0.2$, $\sigma = 1$ for the plots on the bottom. They illustrate that for these parameters it is possible to compress each weight to one bit while keeping the classification error very low. }
    \label{fig:linf_1200_600}
\end{figure}

\end{document}